\documentclass{ecai} 

\usepackage{latexsym}
\usepackage{amssymb}
\usepackage{amsmath}
\usepackage{amsthm}
\usepackage{booktabs}
\usepackage{enumitem}
\usepackage{graphicx}
\usepackage{color}
\usepackage{hyperref}
\usepackage[capitalise]{cleveref}

\usepackage{mathtools}
\usepackage{dsfont} 
\usepackage{bm}
\usepackage{subcaption} 
\usepackage{algorithm}
\usepackage{algpseudocode}

\newtheorem{theorem}{Theorem}
\newtheorem{lemma}[theorem]{Lemma}

\newtheorem{proposition}[theorem]{Proposition}

\newcommand{\BibTeX}{B\kern-.05em{\sc i\kern-.025em b}\kern-.08em\TeX}

\DeclareMathOperator{\E}{\mathbb{E}}

\newif\ifforarxiv
\forarxivtrue

\newcommand{\appendixtext}{%
  \ifforarxiv
    Appendix%
  \else
    supplementary material~\cite{ourpaper-arxiv}%
  \fi
}

\newcommand{\appendixtextnocite}{%
  \ifforarxiv
    Appendix%
  \else
    supplementary material%
  \fi
}

\begin{document}


\begin{frontmatter}


\paperid{0986} 


\title{Aligning the Evaluation of Probabilistic Predictions with Downstream Value}


%
\author[A]{\fnms{Novin}~\snm{Shahroudi}\orcid{}\thanks{Corresponding Author. Email:  novin.shahroudi@ut.ee.}}
\author[A]{\fnms{Viacheslav}~\snm{Komisarenko}\orcid{}}
\author[A]{\fnms{Meelis}~\snm{Kull}\orcid{}}

\address[A]{Institute of Computer Science, University of Tartu}


\begin{abstract}
Every prediction is ultimately used in a downstream task. Consequently, evaluating prediction quality is more meaningful when considered in the context of its downstream use. Metrics based solely on predictive performance often diverge from measures of real-world downstream impact.
Existing approaches incorporate the downstream view by relying on multiple task-specific metrics, which can be burdensome to analyze, or by formulating cost-sensitive evaluations that require an explicit cost structure, typically assumed to be known a priori. We frame this mismatch as an \emph{evaluation alignment problem} and propose a data-driven method to learn a proxy evaluation function aligned with the downstream evaluation. 
Building on the theory of proper scoring rules, we explore transformations of scoring rules that ensure the preservation of propriety.
Our approach leverages weighted scoring rules parametrized by a neural network, where weighting is learned to align with the performance in the downstream task. This enables fast and scalable evaluation cycles across tasks where the weighting is complex or unknown a priori. We showcase our framework through synthetic and real-data experiments for regression tasks, demonstrating its potential to bridge the gap between predictive evaluation and downstream utility in modular prediction systems.
\end{abstract}

\end{frontmatter}


\section{Introduction}
Many real-world applications, especially safety-critical or risk-sensitive ones, rely on probabilistic predictions to manage uncertainty in decision-making. For example, retailers use demand forecasts to optimize inventory and prevent stock-related losses~\cite{halperin2022reinforcement,huber2019data}. Although predictive fidelity, which refers to the accuracy with which predictions capture the true distribution, is crucial, the practical value of predictions ultimately depends on their downstream usefulness. 
Therefore, a systematic procedure is needed to evaluate how well probabilistic predictions support downstream decision-making.

It is well recognized across machine learning, forecasting, and decision science that a model’s predictive quality as measured by standard evaluation metrics does not always translate into decision-making value~\cite{dumas2022deep, hong2020energy, 8464297}. In other words, a model that performs well on a metric such as MAE may not yield the best outcomes when its predictions inform downstream decisions \cite{kappen2018evaluating}. 
This limitation also applies to proper scoring rules~\cite{gneiting2007strictly}, which are considered the gold standard for evaluating probabilistic predictions because they encourage Bayes-optimal predictions. 
Likewise, many downstream tasks assess predictions through expected utility or Bayes risk, typically under the assumption that the predictions are close to optimal.
Ideally, if predictions were optimal, evaluations via proper scoring rules would match the downstream value. However, in practice, suboptimal predictions lead to differences between the two evaluations, reflecting the distinct priorities and asymmetries inherent in each.
Closely related terms include loss–metric mismatch and suboptimality gap~\cite{huang2019addressing, elmachtoub2022smart}. 
In our work, we refer to it as \emph{evaluation misalignment}.

As a toy example, consider three predictions, denoted by $A$, $B$, and $C$, where $A$ is the optimal prediction while $B$ and $C$ are two different suboptimal predictions. Suppose that when evaluated based on predictive fidelity, their ranking is $R_{\text{predictive}}(A,B,C) = 1,2,3$, with $A$ performing the best. However, when evaluated based on downstream value, the ranking may change to $R_{\text{downstream}}(A,B,C) = 1,3,2$. The optimal prediction $A$ consistently ranks first because it is uniquely optimal, but the suboptimal predictions $B$ and $C$ can swap ranks depending on the specifics of the downstream task. 
A simulated illustration of this phenomenon is provided in~\cref{fig:illustration_firstpage}. \\
\begin{figure}[t]
    \centering
    \includegraphics[width=0.45\linewidth]{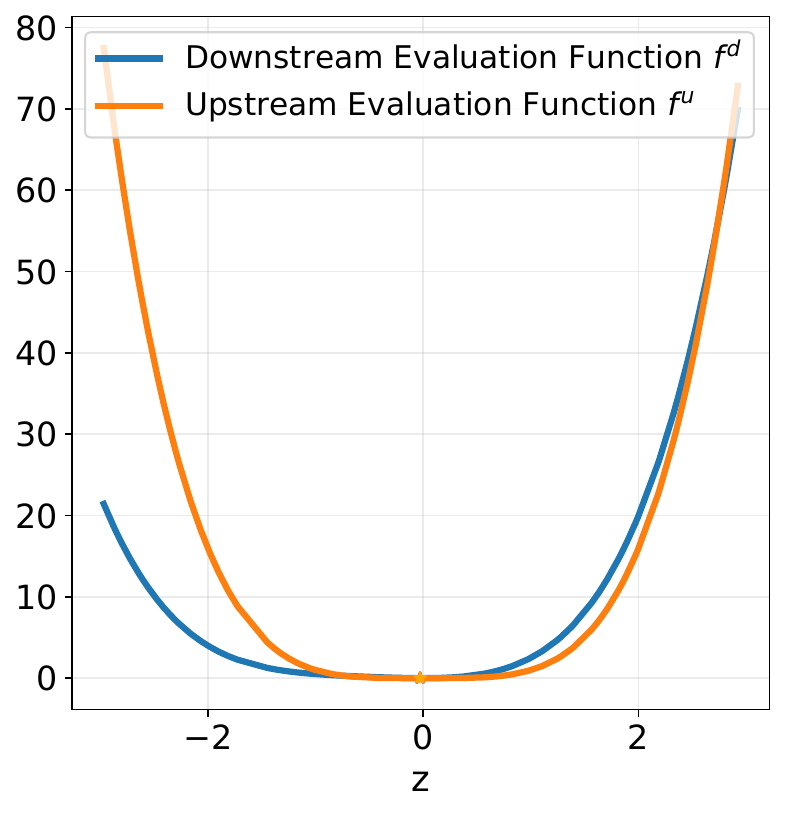}
    \includegraphics[width=0.505\linewidth]{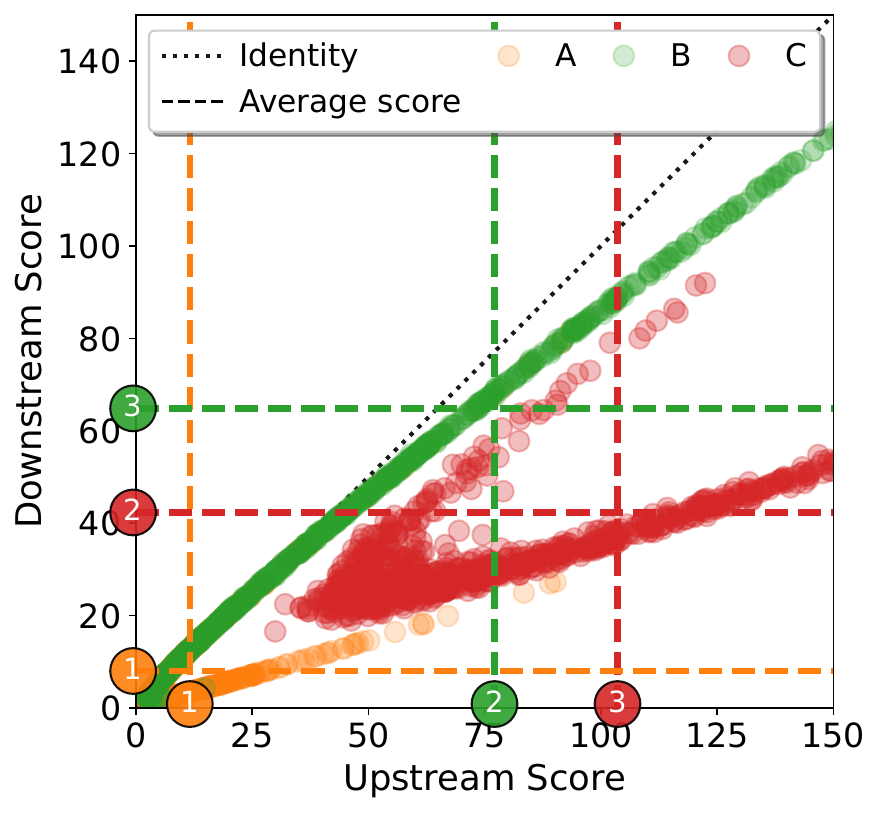}
    \caption{Toy example illustrating the mismatch between predictive quality evaluation (upstream) and evaluation downstream. We evaluate three predictions: $A\!:\!\{\hat{y}^{(j)}_i\}_{j=1}^M\!\sim\!\mathcal{N}(2,1)$, $B\!:\!\{\hat{y}^{(j)}_i\}_{j=1}^M\!\sim\!\mathcal{N}(0,1)$, and $C\!:\!\{\hat{y}^{(j)}_i\}_{j=1}^M\!\sim\!\mathcal{N}(3,2)$, against observations  $y_i\!\sim\!\mathcal{N}(2,1)$ for $i\!\in\{\!1,\dots,\!N\}$ where $N$ is number of instances in the dataset and $M$ samples per predictive distribution. \textbf{Left}: upstream and downstream evaluation functions $f^u(z)\!=\! z^4$ and $f^d(z)\!=\!0.5 (z^4\!+\!2z^3\!+\!2z^2)$, with $z_i\!=\!|y_i\!-\!\hat{y}_i^{(j)}|$. \textbf{Right}: alignment curve; scatter plot of mean scores across $M$ samples, highlighting scale and ranking differences. For visualization, $N\!=\!100$,  $M\!=\!1000$ (typically, $M\!\ll\!N$).} \vspace{1.5em}
    \label{fig:illustration_firstpage}
\end{figure} 

\noindent
When evaluating predictions, two key questions naturally arise:
\begin{enumerate}[nosep]
    \item What is the impact of prediction errors on the downstream task?
    \item Which prediction, among several models, offers the best performance for the downstream task?
\end{enumerate}
\vspace{0.5em}

A direct way to answer these questions would be to evaluate each prediction by running the downstream task. However, executing the downstream task can often be costly and time-consuming. In addition, knowing the impact alone does not directly offer actionable insights about how to improve the predictions. 

Existing studies approach the misalignment from different angles.
One common strategy is to evaluate predictions using a large set of metrics, where each metric captures an aspect relevant to the downstream task, including, when possible, the downstream value itself. However, the multiplicity of metrics makes it difficult to analyze results systematically. Alternatively, end-to-end modeling or decision-focused learning seeks to jointly optimize the prediction and the downstream task, thus aligning the evaluation with use~\cite{mandi2024decision, elmachtoub2022smart, donti2017task, wilder2019melding}. However, this approach is not always desirable. In many cases, such as weather forecasting, we cannot anticipate all potential downstream applications when developing the model. 
Thus, tying predictive modeling to the downstream can limit the scalability of the evaluations, i.e. in order to answer the second key question.

Other approaches address the discrepancy more directly. Cost-sensitive evaluation incorporates domain knowledge about the cost of different types of errors in the assessment. In general, cost-sensitive evaluation has been studied much more for classification~\cite{elkan2001foundations, zadrozny2003cost,ling2008cost,domingos1999metacost, Komisarenko2025} than for regression~\cite{hern2014probabilistic, zhao2011extended}. In many of the existing works, for both classification and regression, the cost structure is explicitly defined.  Weighted scoring rules (WSRs) are a type of cost-sensitive evaluation that offers a principled way to weigh a proper scoring rule in order to account for the downstream impact while preserving the properness~\cite{matheson1976scoring}. In existing studies, the weighting function often has a simple form and it is assumed to be provided by the domain expert~\cite{gneiting2011comparing,holzmann2017focusing,allen2023evaluating}. 
Although cost-sensitive approaches answer both key questions, the quality of their answer depends on domain knowledge about the cost-structure or weight function.

We answer to both key questions by formulating the misalignment problem between predictive performance and downstream value as an \emph{evaluation alignment problem}, aiming to minimize the discrepancy between the two assessments. 
Despite existing works linking scoring rules to decision losses~\cite{dawid2007geometry}, to the best of our knowledge, there has not been any work that would address the evaluation alignment using proper scoring rules. We show how WSRs provide the foundation for the problem and propose a method to learn a weighting using an alignment process between proper scoring rules and downstream decision loss with minimal assumptions about the weight function. 
Unlike decision-focused learning frameworks that align predictions with optimal downstream decisions, resulting in evaluation alignment as a byproduct, we take the reverse approach: We align evaluations, which can then serve as loss functions for training predictive models. Our study focuses exclusively on aligning evaluations, but prediction optimization is a potential use case of our method. \\

\vspace{-0.5em}
\noindent
Our contributions are as follows: 
\begin{enumerate}[nosep]
    \item We first introduce the mismatch between predictive quality and downstream value of probabilistic predictions as an evaluation alignment problem. 
    \item We then characterize the class of transformations that preserve the propriety of scoring rules, thus providing a principled foundation for aligning evaluations through transformed scoring rules.
    \item And we propose a neural-network-based alignment method that parameterizes a WSR to learn a proxy evaluation function aligned with the downstream value.
    \item Finally, we showcase our approach empirically on synthetic and real-data experiments for univariate regression tasks.
\end{enumerate}
\vspace{0.5em}

The remainder of the paper is organized as follows. In \cref{sec:preliminary}, we introduce key definitions and background. \Cref{sec:alignment} formalizes the evaluation alignment problem, and \cref{sec:method} describes our proposed methodology. Experimental results are presented in \cref{sec:experiments}, followed by concluding remarks in \cref{sec:conclusion}.


\section{Preliminaries}\label{sec:preliminary}
This section introduces components necessary for our alignment framework, which centers on decision-making under uncertainty informed by probabilistic predictions. 
Let \( Y \) denote a random variable representing the target quantity of interest, with true distribution \( P_Y \). We observe realizations \( y \sim P_Y \), such as a number of umbrellas sold in a store on a rainy day, and construct a predictive distribution based on the observations. 

\textit{\textbf{Upstream task}} refers to a probabilistic modeling task that aims to capture the uncertainty in predicting the target variable $Y$. Instead of issuing point estimates, the model outputs a predictive cumulative distribution function (CDF) $\hat{P}_Y$ to approximate the true target distribution $P_Y$.
The output of the upstream task is ultimately used as an input to a \textit{downstream task}. Thus, we refer to it as "upstream" to juxtapose it with "downstream." 

\textit{\textbf{Sample-based prediction}} In this work, we consider predictions to be represented in the form of samples from the predictive distribution. That is, instead of working with an explicit parametric form for \( \hat{P}_Y \), we approximate it via \( M \) samples:
\[
\hat{\mathbf{y}} = \left\{ \hat{y}^{(j)} \right\}_{j=1}^M \sim \hat{P}_Y.
\]
This sample-based representation is common in probabilistic models such as Bayesian neural networks and variational inference, which rely on Monte Carlo sampling where direct access to the density or CDF of \( \hat{P}_Y \) may not be available~\cite{murphy2012machine}. 
\textit{Proper scoring rules} provide a principled way to assess the quality of probabilistic predictions by incentivizing accurate and calibrated uncertainty, thus, we adopt them as a standard way to evaluate the upstream task. We next define what a scoring rule is and what it means to be proper. 

\textit{\textbf{Scoring rule}} is a function \( S(\hat{P}_Y, y) \in \mathbb{R} \) that evaluates the quality of a probabilistic prediction \( \hat{P}_Y \) given the realized outcome \( y \). A scoring rule is called \textit{proper} if it encourages truthful reporting of uncertainty, i.e., if the expected score is minimized when the predictive distribution matches the true distribution:
\[
\mathbb{E}_{y \sim P_Y}[S(\hat{P}_Y, y)] \geq \mathbb{E}_{y \sim P_Y}[S(P_Y, y)] \quad \text{for all } P_Y,\hat{P}_Y.
\]
It is \textit{strictly proper} if the inequality becomes an equality if and only if \( \hat{P}_Y = P_Y \). That is, the expected score
\[
S(\hat{P}_Y, P_Y) := \mathbb{E}_{y \sim P_Y}[S(\hat{P}_Y, y)]
\]
satisfies \( S(\hat{P}_Y, P_Y) \geq S(P_Y, P_Y) \), with equality if and only if \( \hat{P}_Y = P_Y \). See \cite{gneiting2007strictly} for a comprehensive treatment of proper scoring rules. \\

When the predictive distribution \( \hat{P}_Y \) is available only through samples \( \hat{\mathbf{y}} = \{ \hat{y}^{(j)} \}_{j=1}^M \), proper scoring rules are evaluated using sample-based approximations. 
Some scoring rules such as CRPS have an unbiased closed-form estimator~\cite{matheson1976scoring, gneiting2005crps} that can be readily applied to the samples. We denote the scoring rule operating on samples by \( S(\hat{\mathbf{y}}, y) \), where \( \hat{\mathbf{y}}\) is a vector of $M$ prediction samples drawn i.i.d. from \( \hat{P}_Y \), and \( y \sim P_Y \) is the observed outcome~\footnote{We slightly abuse notation by using the same symbol \( S \) for both the theoretical scoring rule defined on distributions and its sample-based estimator; the distinction will be clear from the context.}. Other scoring rules such as \emph{log score} may not admit closed-form sample-based estimators and require kernel density estimation or alternative approximations \cite{tagasovska2019single}. 

We consider a dataset consisting of \( N \) instances, indexed by \( i = 1, \dots, N \), where each instance corresponds to a realization \( y_i \sim P_{Y_i} \) of the target variable and the associated input features (omitted here for brevity). The upstream model issues a probabilistic prediction \( \hat{P}_{Y_i} \) for each instance, representing an approximation of the conditional distribution of \( Y_i \) given the features. These predictive distributions serve two roles: they are evaluated directly via proper scoring rules (upstream evaluation), and they are used as input to a downstream decision-making process that induces a task-specific cost or profit (downstream evaluation). Each instance, therefore, gives rise to both an upstream and downstream score, which we seek to align, that we will discuss in~\cref{sec:alignment}. Therefore, the average upstream score over the dataset is calculated as 
\begin{equation*}
    \frac{1}{N} \sum_{i=1}^N S(\mathbf{\hat{y}}_i, y_i),
\end{equation*}
where $\hat{\mathbf{y}}_i = \{\hat{y}^{(j)}_i\}_{j=1}^M$ for each instance $i$.

\textit{\textbf{Downstream task}} refers to a mathematical optimization problem under uncertainty, where a decision-maker uses the predictive distribution \( \hat{P}_Y \), along with other problem-specific parameters and constraints, to compute optimal decisions. In our work, we bring examples about and experiment on \emph{stochastic programming}~\cite{shapiro2021lectures} as an instance of decision-making under uncertainty. The downstream objective is to find a decision \( \hat{a} \in \mathcal{A} \) that minimizes the expected cost with respect to the predictive distribution \( \hat{P}_Y \):
\begin{align}
    \label{eq:downstream_opt_action}
    \hat{a} &= \mathop{\arg\max}\limits_{a\in\mathcal{A}} \mathbb{E}_{\hat{y}\sim \hat{P}_Y}[\pi(a, \hat{y})],
\end{align}
where \( \pi(a, \hat{y}) \) denotes the task-specific profit\footnote{Alternatively, \(\pi(a,\hat{y})\) could denote a loss instead of profit, in which case the above objective would be a minimization instead of maximization.} obtained when taking decision \( a \) under realization \( \hat{y} \). To evaluate the performance of the decision \( \hat{a} \), we compute the expected profit under the true distribution \( P_Y \):
\begin{align}
    \label{eq:downstream_score}
    \mathbb{E}_{y\sim P_Y}[\pi(\hat{a}, y)].
\end{align}

\Cref{eq:downstream_opt_action} defines the optimization step based on the model’s belief (\( \hat{P}_Y \)), while~\cref{eq:downstream_score} reflects the actual performance under the true data-generating process (\( P_Y \)). 
\\

The \textit{\textbf{upstream score}} \( s^u_i \) measures the quality of the prediction \( \hat{\mathbf{y}}_i \) against observation $y_i$ for the \( i^{\text{th}} \) instance using a proper scoring rule:
\[
s^u_i = S(\hat{\mathbf{y}}_i, y_i).
\]

The \textit{\textbf{downstream score}} \( s^d_i \) captures the quality of the decision \( \hat{a}_i \), which is obtained by solving a stochastic optimization problem using the predictive distribution \( \hat{P}_{Y_i} \). It is defined as:
\[
s^d_i = \pi(\hat{a}_i, y_i),
\]
where \( \pi(\cdot, \cdot) \) denotes the task-specific profit function, and \( \hat{a}_i \) is the optimal decision computed under \( \hat{P}_{Y_i} \).

The corresponding average (dataset-level) scores are estimated as follows:
\begin{align*}
    \bar{s}^u &= \frac{1}{N} \sum_{i=1}^N s^u_i= \frac{1}{N} \sum_{i=1}^N S(\hat{\mathbf{y}}_i, y_i), \\
    \bar{s}^d &= \frac{1}{N} \sum_{i=1}^N s^d_i = \frac{1}{N} \sum_{i=1}^N \pi(\hat{a}_i, y_i).
\end{align*}

\textit{\textbf{Continuous Ranked Probability Score (CRPS)}} is a common proper scoring rule to assess univariate continuous predictions, defined as follows:
\begin{equation}\label{eq:crps}
    \begin{aligned}
    \text{CRPS}(P, y) &= \int_{\mathbb{R}} (P(z) - \mathds{1}\{y \leq z\})^2 \,dz \\
    &= \E_P |X - y| - \frac{1}{2} \E_P |X - X'|,
    \end{aligned}
\end{equation}
where $y \in \mathbb{R}$ is the observation, and $X, X' \sim P$ are independently drawn from the predictive distribution \(\hat{P}_Y\) with the cumulative distribution function $P$.
The second line of Eq~\ref{eq:crps} corresponds to the kernel form~\cite{allen2023evaluating, gneiting2007strictly}. The sample-based approximation is immediately applicable to the kernel form and can be estimated as 
\begin{equation}\label{eq:crps_sample_based}
    \begin{aligned}
    \widehat{CRPS}(\mathbf{\hat{y}}, y) &= \sum_{j=1}^M |\hat{y}^{(j)} - y| - \frac{1}{2M^2} \sum_{j=1}^M\sum_{k=1}^M |\hat{y}^{(j)} - \hat{y}^{(k)}|,
    \end{aligned}
\end{equation}
and the dataset-wide average score can be calculated as 
\begin{equation*}
    \overline{CRPS} = \frac{1}{N} \sum_{i=1}^N \widehat{CRPS}(\mathbf{\hat{y}}_i, y_i)
\end{equation*}

Weighted scoring rules (WSRs) extend ordinary proper scores by attaching greater importance to regions of the outcome space that matter most to the user.  A convenient construction is \emph{threshold weighting} \citep{gneiting2011comparing}, which
integrates the underlying score against a non‑negative weight
\(w\) (or its strictly increasing \emph{chaining} primitive \(v\)).  For the
continuous ranked probability score (CRPS), this yields
\begin{equation}\label{eq:twcrps}
\begin{aligned}
    \mathrm{twCRPS}(P,y;w)&=
    \int_{\mathbb R}\bigl(P(z)-\mathbf 1\{y\le z\}\bigr)^{2}w(z)\,dz \\
    =\, &
    \mathbb E_{P}\!\bigl|v(X)-v(y)\bigr|
      -\tfrac12\mathbb E_{P}\!\bigl|v(X)-v(X')\bigr|,
\end{aligned}
\end{equation}
which is proper, and it is strictly proper whenever \(w\) is strictly positive.
A Monte‑Carlo (ensemble) estimator is obtained by replacing the
expectations with averages over predictive samples
\(\hat y^{(1)},\dots,\hat y^{(M)}\).

Recent work by Allen \emph{et al.} \citep{allen2023evaluating,allen2024weighted}
systematises threshold weighting and provides guidance on choosing \(w\) or \(v\) to emphasise, for example, exceedances above a policy‑relevant threshold.
We refer the reader to those papers for detailed derivations, recipes, and
additional examples.


\section{Evaluation Alignment}\label{sec:alignment}

In many predictive settings, the objective used to train a model (the upstream score) differs from the metric that drives real-world decisions (the downstream score), leading to a mismatch of upstream and downstream regressor rankings. In this section, we formalise \textit{evaluation alignment} - the process of transforming the upstream score so that it agrees with the downstream criterion - and show why it is both necessary and achievable. Building on proper scoring rule theory, we frame alignment as minimizing the discrepancy between two scoring rules via monotone re-parameterizations. We then prove that for some families of losses, including CRPS and its threshold-weighted variants, this procedure yields perfect alignment.

In most downstream evaluations, one must first make a decision under uncertainty, selecting an action without yet knowing the true outcome. By Bayesian decision theory, a rational approach is to choose an action that minimises the forecast-conditional expectation of the downstream loss. 
More formally, assume that for each predictive distribution \(P\) the Bayes act
\[
a^{\star}(P)=\arg\min_{a\in\mathcal A}\mathbb E_{Y\sim P}\,\pi(a,Y)
\]
exists and is unique (e.g.\ when \(a\mapsto\pi(a,y)\) is strictly convex).  Then the induced score
\[
s^d(P,Y)\coloneqq\pi\bigl(a^{\star}(P),Y\bigr)
\]
is (strictly) proper with respect to \(P\) and \(Y\) \cite{savage1971elicitation,gneiting2007strictly}.
Selecting an action that minimises expected utility is the backbone of many decision‑theoretic scoring constructions \citep{winkler1969scoring}. The properness follows almost immediately from the optimisation principle itself, yet this point is typically noted only in passing in the literature, for instance by \citep{gneiting2007strictly}.

Because our focus is on aligning evaluation metrics, we leave the exact action space unspecified.  We assume only that, for every predictive distribution \(P\), a Bayes-optimal action \(a^{\star}(P)\) exists and (conceptually) will be taken.  We therefore fold this decision step into the downstream loss and write
\[
s^{\mathrm d}(P,Y)\;=\;\pi\bigl(a^{\star}(P),Y\bigr),
\]
so that \(s^{\mathrm d}\) reports the realised loss \emph{after} optimal action. 

The idea behind \textit{\textbf{Evaluation Alignment}} is to adjust the upstream evaluation such that it produces the same evaluation as the downstream. We refer to such an evaluation as an \emph{aligned evaluation} and refer to the process of obtaining an aligned evaluation as \textit{alignment}. 

To achieve such alignment, we introduce transformations of the upstream loss, namely, by re-mapping its arguments (the forecast and the realised outcome) and/or its value.  
Guided by the Proposition below,
we restrict the transformation to a strictly increasing bijection \(\nu\) applied identically to both the forecast and the outcome, together with an injective transformation \(h\) on the loss itself - two ingredients that \emph{preserve}
(strict) propriety:
\[
    s_{\rm aligned}(P,y)\;=\;h\!\bigl(s^u\bigl(P\circ \nu^{-1},\,\nu(y)\bigr)\bigr).
\]
Here the function $P\circ \nu^{-1}$ is chosen such as it is the function that reparametrizes $P$ to act in the space transformed by $\nu$, i.e. $(P\circ\nu^{-1})(\nu(y))=P(y)$ for any $y$.
By using transforms that maintain properness, we preserve the upstream score’s incentive towards truthful forecasts.

\begin{proposition}
Let \(L(P,y)\) be a (strictly) proper scoring rule on an outcome space \(\mathcal Y\subseteq\mathbb R\).  Define
\[
S(P,y)\;=\;h\!\bigl(L\bigl(P\circ \nu^{-1},\nu(y)\bigr)\bigr),
\]
where \(\nu:\mathcal Y\to\mathcal Y\) and \(h:\mathbb R\to\mathbb R\) are continuous functions.  Then \(S\) is (strictly) proper if and only if:
\begin{enumerate}
  \item \(\nu\) is a bijection on \(\mathcal Y\), and
 \item \(h(s)=a\,s+b\) is an affine function with \(a>0\).
\end{enumerate}
No other choice of \((h, \nu)\) preserves propriety.
\end{proposition}

\begin{proof}[A complete proof is presented in the \appendixtext]

\end{proof}

Perfect alignment is not always attainable.
To construct an example demonstrating this, let us fix the predictive distribution \(P=\mathrm{Unif}(0,1)\).
Then the logarithmic score is identically zero,
\(S_{\log}(P,y)=-\log p(y)=0\) for all \(y\in(0,1)\).
By contrast, the CRPS for the same uniform \(P\) is  
\[
\begin{aligned}
S_{\mathrm{CRPS}}(P,y)
&= \int_{0}^{y} z^{2}\,dz \;+\;\int_{y}^{1}(1-z)^{2}\,dz
\;=\;\frac{y^{3}+(1-y)^{3}}{3}\\
&= y^{2}-y+\tfrac{1}{3},\qquad 0\le y\le 1
\end{aligned}
\]
which is a non-constant quadratic.  Applying any injective transform \(h\) to
\(S_{\log}\) still produces the constant \(h(0)\); hence no monotone \(\nu\)  and injective \(h\) can convert
\(S_{\log}\) into this varying CRPS.

However, for a particular family of proper scoring rules, it is possible to demonstrate perfect alignment. 

\begin{lemma}

Let 
\[
s^u(P,y)=\int_{-\infty}^{\infty}w^u(x)\;k\bigl(P(x),\mathbf1\{y\le x\}\bigr)\,dx,
\]
\[
s^d(P,y)=\int_{-\infty}^{\infty}w^d(x)\;k\bigl(P(x),\mathbf1\{y\le x\}\bigr)\,dx
\]
be two strictly proper integral scoring rules sharing the same strictly proper binary loss \(k\) with strictly positive continuous weight functions $w^u$, $w^d$ which both integrate to $1$, i.e. have unit total mass.  There exists a strictly increasing bijection \(\nu:\mathbb R\to\mathbb R\) and an affine function \(h:\mathbb R\to\mathbb R\) such that
\[
s^d(P,y)=h\!\bigl(s^u(P\circ \nu^{-1},\nu(y))\bigr)
\quad\forall\,P,y,
\]

\end{lemma}

\begin{proof}[Proof sketch]
Define the cumulative weight functions
\[
W^u(t)=\int_{-\infty}^t w^u(x)\,dx,
\quad
W^d(t)=\int_{-\infty}^t w^d(x)\,dx.
\]
Because \(w^u\) and \(w^d\) are positive and integrable, each \(W^\bullet\) is a continuous, strictly increasing bijection from \(\mathbb R\) onto \([0,1]\).  Consequently, its inverse \((W^u)^{-1}:[0,1]\to\mathbb R\) exists, and the map
\[
\nu(t) \;=\; (W^u)^{-1}\bigl(W^d(t)\bigr)
\]
is itself a continuous, strictly increasing bijection on \(\mathbb R\).

By a change of variables \(x=\nu(t)\) one shows
\[
s^u\bigl(P\circ \nu^{-1},\nu(y)\bigr)
=\int_{-\infty}^{\infty}w^u\bigl(\nu(t)\bigr)\,\nu'(t)\;k\bigl(P(t),\mathbf1\{y\le t\}\bigr)\,dt.
\]
Let \(h(s)=s\), then, given that the identity  
\(W^{u}(\nu(t)) = W^{d}(t)\) implies, by differentiation,
$
w^{u}\!\bigl(\nu(t)\bigr)\,\nu'(t)=w^{d}(t)\quad\forall t\in\mathbb R .
$
Then the expression is the following:
\[
\begin{aligned}
h\bigl(s^{u}(P\circ \nu^{-1},\nu(y))\bigr)
  = \int_{-\infty}^{\infty}\! w^{u}\!\bigl(\nu(t)\bigr)\,\nu'(t)\,
       k\!\bigl(P(t),\mathbf 1\{y\le t\}\bigr)\,dt \\[2pt] = \int_{-\infty}^{\infty}
   w^{d}(t)\,k\!\bigl(P(t),\mathbf 1\{y\le t\}\bigr)\,dt
   \;=\; s^{d}(P,y),
\end{aligned}\]
which means we reconstructed the desired loss $s^d$.
\end{proof}

The considered family of losses is sufficiently rich as it includes all binary scoring rules as building blocks.
A CRPS is a special case of this family, with $k(x, y)=(x-y)^2$. Its transformation to threshold-weighted CRPS using monotonic $\nu$ was demonstrated in \cite{allen2024weighted}. 

Alignment is generally non‑unique.  For instance, when CRPS aligns with
itself, the identity pair $(h,\nu)$ is only one
solution; the whole family
\(\nu_{a,c}(y)=a\,y+c,\;h_{a,c}(s)=s/a\;(a>0,c\in\mathbb R)\)
also leaves the score unchanged:
\(h_{a,c}\!\bigl(s^{u}\circ \nu_{a,c}^{-1}\bigr)=s^{u}\).

The evaluation alignment could be formulated as an optimization problem that tries to minimize distance between upstream and downstream scores. As demonstrated above, for certain losses the alignment can be perfect, and hence, the distance would become zero.
Thus, we frame the alignment problem as a learning task to find transformations $h$ and $\nu$ that adjust the upstream evaluation and yield an aligned evaluation. 

Therefore, the alignment objective is 
\begin{equation}
     \mathop{\arg\min}\limits_{h\in\mathcal{H}, \nu\in\mathcal{V}} \E_i \delta(\hat{s}^d_i, s^d_i),
\end{equation}
where \(\delta\) denotes a distance metric (e.g., mean-squared error).  
The transformed score \(\hat{s}^d\) is regarded as an improvement over the
original upstream score \(s^{u}\) whenever
\begin{equation}
    \E_i \delta(\hat{s}^d_i, s^d_i) \leq \E_i \delta(s^u_i, s^d_i),
\end{equation}
with equality indicating that $s^{u}$ already achieves the minimal
expected distance attainable within the chosen transformation classes.
Intuitively, we seek two \emph{global} transformations\footnote{%
'Global’ means the same mappings are applied to every instance $i$; the
functions are not constant in the mathematical sense $f(x)=c$, but
instance--independent once learned.}
\(h\in\mathcal{H}\) and \(\nu\in\mathcal{V}\) that minimise the average
discrepancy between the transformed upstream score
\(\hat{s}^{d}=h\!\bigl(s^{u}\circ v^{-1}\bigr)\) and the true downstream score
\(s^{d}\).


\section{Alignment Model}\label{sec:method}

In this section, we present a learning-based approach to evaluation alignment, where we train a neural network to approximate an aligned score that better reflects downstream utility. 

We use a neural network $f$ with parameters $\bm{\theta}$ to parameterize the transformations $\nu$ and $h$ of a scoring rule $S$ as depicted in the diagram in Fig.~\ref{fig:alignment_pipeline} and refer to it as the \emph{alignment model}. The scoring rule $S$ functions as an operator within the architecture and, in this context, can be treated as a fixed design choice or hyperparameter.
The neural network takes the samples $\hat{\mathbf{y}}$ from the predictive CDF $\hat{P}_Y$ and observation $y$ as input and estimates an aligned score, 
\begin{equation} \label{eq:parameterized_tw_scoringrule}
    \hat{s}^d = f_{S,\bm{\theta}}(\hat{\mathbf{y}}, y).
\end{equation}

\noindent
Since the downstream score $s^d$ serves as the target, the alignment task can be naturally framed as a regression problem. The objective of the neural network is thus to minimize the discrepancy between its predicted score and the true downstream score:
\begin{equation}\label{eq:alignment_loss}
    \mathop{\arg\min}\limits_{\bm{\theta}} |f_{S, \bm{\theta}}(\hat{\mathbf{y}}, y) - s^d|^2,
\end{equation}
with $\bm{\theta}=(\theta_1, \theta_2)$, where $\theta_1$ and $\theta_2$ correspond to the parameters of the input and output transformations, respectively. We refer to Eq.~\ref{eq:alignment_loss} as \emph{alignment loss}.

\begin{figure}[ht]
    \centering
    \includegraphics[width=0.95\linewidth]{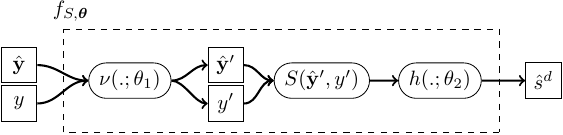}
    \vspace{1em}
    \caption{Dataflow diagram for the alignment model. It takes predictions $\hat{\mathbf{y}}$, observations $y$, and it outputs the estimated downstream score $\hat{s}^d$. It performs two transformations on the $S$. First, $\nu$ that applies a monotonic mapping on the inputs of $S$, and second, $h$ that performs a linear mapping on the output of $S$.}\vspace{2em}
    \label{fig:alignment_pipeline}
\end{figure}

The alignment model is trained in a post-hoc manner using a validation (or hold-out) set from the predictive task, which we refer to as the \emph{alignment set}, as detailed in Algorithm~\ref{alg:alignment_train}. 

\renewcommand{\algorithmicrequire}{\textbf{Input:}}
\renewcommand{\algorithmicreturn}{\textbf{Output:}}

\begin{figure*}[t]
\centering
    \begin{minipage}[t]{0.48\textwidth}
        \begin{algorithm}[H]
            \caption{Alignment Model Training}
            \label{alg:alignment_train}
            \begin{algorithmic}[1]
                \Require Validation set $\mathcal D_{\rm val}\!=\!\{(\hat{\mathbf{y}}_i,y_i,s^d_i)\}_{i=1}^{N_{val}}$, model $f_{S,\theta}$, alignment loss $\ell$, optimizer (e.g.\ SGD)
                \State Split $\mathcal D_{\rm val}\!\to\!\mathsf D_{\rm train},\mathsf D_{\rm val}$
                \State $\displaystyle\theta^* \leftarrow \arg\min_\theta \E_{(\hat{\mathbf{y}},y,s)\in\mathsf D_{\rm train}}\bigl[\ell\bigl(f_{S,\theta}(\hat{\mathbf{y}}_i,y),\,s\bigr)\bigr]$
                \State (Optionally: monitor $\E_{\mathsf D_{\rm val}}[\ell(f_{S,\theta}(\hat{\mathbf{y}}_i,y),s)]$ for early-stopping or hyperparameter tuning)
                \State \Return $f_{S,\theta^*}$ (trained model)
            \end{algorithmic}
        \end{algorithm}
    \end{minipage}
        \hfill
    \begin{minipage}[t]{0.48\textwidth}
        \begin{algorithm}[H]
            \caption{Alignment Model Inference and Evaluation}
            \label{alg:alignment_inference}
            \begin{algorithmic}[1]
                \Require Test set $\mathcal D_{\rm test}\!=\!\{(\hat{\mathbf{y}}_i,y_i,s^d_i)\}_{i=1}^{N_{test}}$, Trained model $f_{S,\theta^*}$
                \State $\displaystyle \hat s^d_i \leftarrow f_{S,\theta^*}(\hat{\mathbf{y}}_i,\,y_i)\quad\forall i$
                \State $\displaystyle \mathrm{MAE}\leftarrow\frac1{N_{test}}\sum_{i=1}^{N_{test}}|\hat s^d_i - 
                s^d_i|$
                \State \Return Predicted scores $\{\hat s^d_i\}_{i=1}^N$, $\mathrm{MAE}$
            \end{algorithmic}
        \end{algorithm}
    \end{minipage}
\end{figure*}

\subsection{Evaluating the Alignment}\label{sec:alignment_eval_metrics}

We consider a distance-based and a rank-based metric to assess the goodness of the alignment.
In distance-based assessment, we use the mean absolute error between the estimated downstream score $\hat{s}^d$ and the downstream score $s^d$: 
$$
mae(\hat{s}^d, s^d) = |\hat{s}^d - s^d|
$$

With a perfect alignment $mae(\hat{s}^d, s^d) \to 0$. \\

In the \textit{rank-based assessment}, we evaluate how well the upstream evaluation agrees with the downstream evaluation in ranking different predictive models. Specifically, we use the Kendall tau correlation coefficient (\( \tau \)) to quantify the rank agreement between the two across test instances. The coefficient ranges from \( -1 \) (complete disagreement) to \( 1 \) (perfect agreement), with \( \tau = 0 \) indicating no correlation between the rankings. A higher \( \tau \) value reflects stronger agreement, and perfect alignment is achieved when \( \tau = 1 \).

\Cref{alg:alignment_inference} outlines the inference and evaluation procedure for the trained alignment model. Given a test set containing predictive samples, observations, and corresponding downstream scores, the model produces aligned score estimates \(\hat{s}^d\) for each instance. The quality of these predictions is then assessed using the $mae$ and $\tau$.

\subsection{Architecture}

The neural-based weighted scoring rule represents the weighting function for a scoring rule by a neural network primarily using monotonic and linear layers. We adopt the monotonic layer from~\cite{runje2023constrained} to ensure monotonicity over transformations $\nu$ and $h$. 
The architecture is depicted in Fig.~\ref{fig:architecture} comprised of two blocks, namely, the input-transformation block, parameterizing the transformation function $\nu$, which is responsible for a monotonic map of the inputs to the scoring rule $S$, and the output-transformation block, parameterizing the monotonic transformation $h$ of the scoring rule's output.  

Scoring rule $S$ is passed as a function to the model and acts as an operator within the architecture. Since we intend to use the same transformation for both inputs to the scoring rule, we use the same input-transformation block, i.e., the same transformation for both prediction and observation. 
That means, for each input instance, $M\!+\!1$ forward passes are applied using the input-transformation block because there are $M$ predictive samples and one observation for each input to $S$. And for the output-transformation block, only one forward pass is applied using this block, since $S$ returns a single real value.

\begin{figure}[ht]
    \centering
    \includegraphics[width=0.99\linewidth]{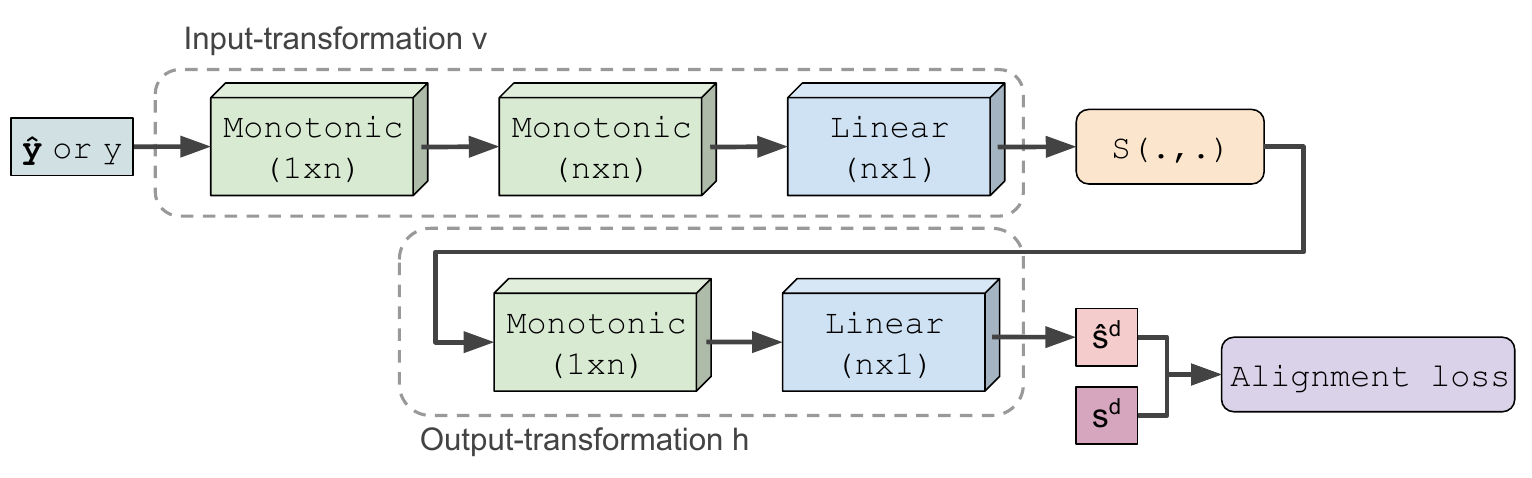}
    \caption{Alignment model architecture consists of monotonic and linear layers. Scoring rule $S$ is passed to the network as an operator within the neural network. The preceding layers of the scoring rule perform transformation $\nu$ to the scoring rule's input, and the proceeding layers perform tranformation $h$ to the scoring rule's output. 
    }\vspace{2em}
    \label{fig:architecture}
\end{figure}

The input features to the pipeline are one-dimensional vectors of $\hat{\mathbf{y}}$ and $y$; thus, the input layer has the dimensionality of $1\times n$ where $n$ is the size of the hidden layer. Similarly, the output of the operator $S$ is a one-dimensional real value; thus, the input layer of the output-transformation block is also $1 \times n$. The number of hidden layers for each block determines the complexity of the monotonic transformation required, i.e., the number of inflection points needed, and can be considered as a hyperparameter alongside the number of weights per layer. \\

Since the monotonic layers output in a normalized range of $[-1, 1]$, we conclude the monotonic transformations with a linear transformation in the input-transformation block to accommodate for unbounded real values.


\section{Experiments and Results}\label{sec:experiments}

We considered two sets of experiments one with synthetic downstream where we generate the downstream value from a known function, and second experiment with an actual downstream task. The former is a type of sanity check for our method, while the latter puts it into test in a realistic setup. Additional details and experiments can be found in~\cref{sec:appdx_exp_convex,sec:appdx_exp_synth_downstream_ext,sec:appdx_exp_downstream_realdata} of the~\appendixtext. Code used for the experiments is available in our Github repository~\cite{githubrepo}. 

\subsection{Synthetic Downstream}\label{sec:syn_regression}
For this experiment we follow the steps: (1) create a synthetic regression dataset, (2) fit a probabilistic model on this data and then do inference to obtain $\hat{P}_Y$ (see~\cref{alg:predictive_train,alg:predictive_inference} in the~\appendixtextnocite), (3) instead of a real downstream decision making task, we adopt a weighted scoring rule, namely, the threshold-weighted CRPS ($twCRPS$) and generate target downstream scores $s^d$ based on some fixed chaining function $\nu$, (4) use the procedure from~\cref{alg:alignment_train} to train the alignment model to learn the weight function $\nu$ and obtain the estimated downstream score $\hat{s}^d$. 
Here, we focus on Steps 3 and 4. For more details on Steps 1 and 2, see~\cref{sec:appdx_exp_synth_downstream_ext} in the~\appendixtext.

In Step 3, instead of having an actual downstream task, we adopt the kernel form of $twCRPS$ from~\cref{eq:twcrps} as the downstream evaluation function and generate synthetic downstream scores with some fixed chaining function $\nu$. To the alignment model the weight function (and the respective chaining function) are unknown, but this setup allows us to have a ground truth for the weight function which can be seen as a sanity check for the alignment pipeline.
For the chaining function $\nu$, we consider different functions as reported in~\cref{tab:chaining_table}. Chaining functions I-III are based on the ones introduced in~\cite{allen2023evaluating} while we devised chaining function IV as a more expressive form. An illustration of each function type can be found in~\cref{fig:synth_chaining} of the~\appendixtext. 

\begin{table}[ht]
    \centering
    \caption{Chaining functions considered for the synthetic experiment.}
    \begin{tabular}{ccl}
         & Name & Chaining Function \\\hline
        I & Threshold & $\nu(z;t) = max(z, t)$ \\
        II & Interval & $\nu(z;a,b) = min(max(z, a), b); a < b$ \\
        III & Gaussian & $\nu(z;t,\mu,\sigma) = (z-t)  \Phi_{\mu, \sigma}(z) + \sigma^2 \phi_{\mu, \sigma}(z)$ \\
        IV & SumSigmoids & $\nu(z;\mathbf{a},\mathbf{b},\mathbf{c},\mathbf{d}) = \sum_i \frac{c_i}{1+\exp{(-(a_iz + b_i))}+d_i}$ \\
    \end{tabular}
    \label{tab:chaining_table}
\end{table}

In Step 4, we use the same kernel form of twCRPS from~\cref{eq:twcrps} that we generated the downstream target scores with, but now it is parametrized using a neural network $f$ with parameters $\bm{\theta}$ according to~\cref{eq:parameterized_tw_scoringrule}. 
We follow~\cref{alg:alignment_train} to learn optimal parameters for the alignment model. Thanks to this setup we can compare $\hat{\nu} = \nu(.,\hat{\theta}_1)$ with its true ground truth $\nu$ chosen in the generation step. To obtain $\hat{s}^d_i$ we follow~\cref{alg:alignment_inference} for inference. 

For this experiment, our method can find near-perfect alignment as depicted in~\cref{fig:res_synth_reg_short}. Thus, we do not report the evaluation metrics for this experiment, as the errors are extremely close to zero. 
This figure shows the weighting function recovered almost perfectly by the alignment model. By definition, since $\frac{d}{dz}(\nu(z)+c) = w(z)$, the alignment model learns $\hat{\nu} = \nu + c$ with an arbitrary constant $c$, leading to an offset in $\hat{\nu}$ but this does not make a difference for the alignment.  
For additional implementation details see~\cref{sec:appdx_exp_synth_downstream_ext} in the~\appendixtext.

\begin{figure*}[htbp]
    \centering
    \begin{subfigure}[b]{0.24\textwidth}
        \centering
        \includegraphics[width=\textwidth]{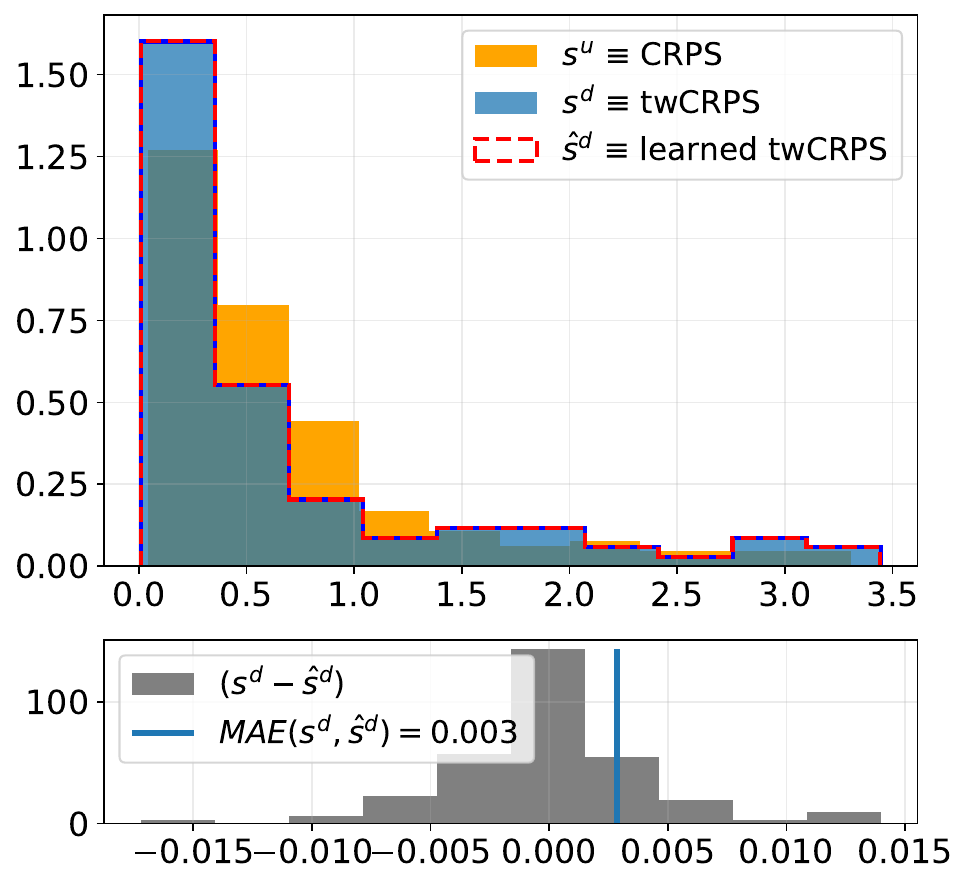} 
    \end{subfigure}
    \begin{subfigure}[b]{0.23\textwidth}
        \centering
        \includegraphics[width=\textwidth]{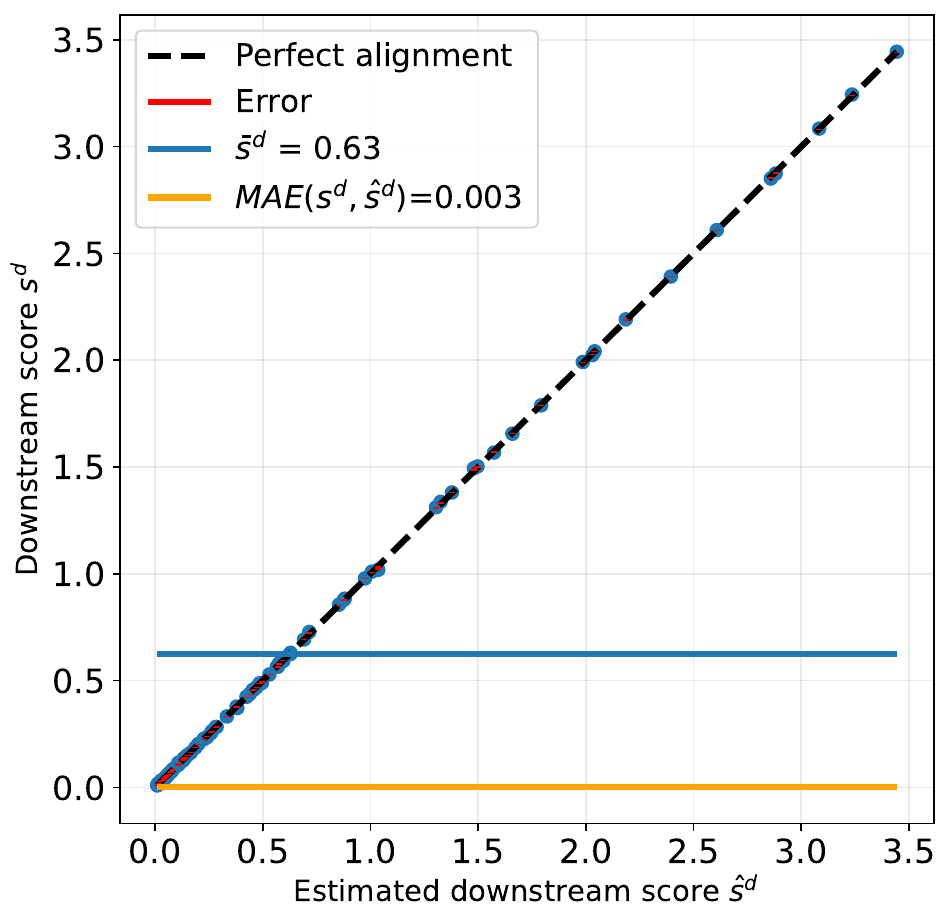} 
    \end{subfigure}
    \begin{subfigure}[b]{0.24\textwidth}
        \centering
        \includegraphics[width=\textwidth]{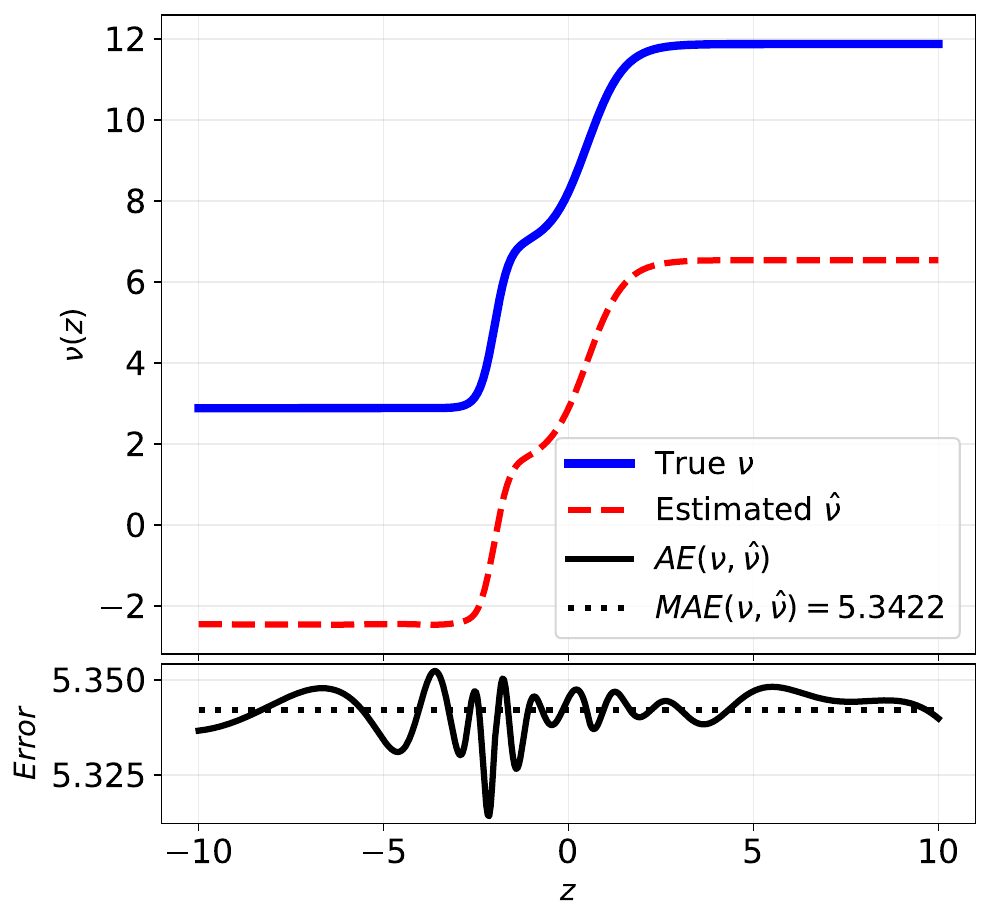} %
    \end{subfigure}
    \begin{subfigure}[b]{0.24\textwidth}
        \centering
        \includegraphics[width=\textwidth]{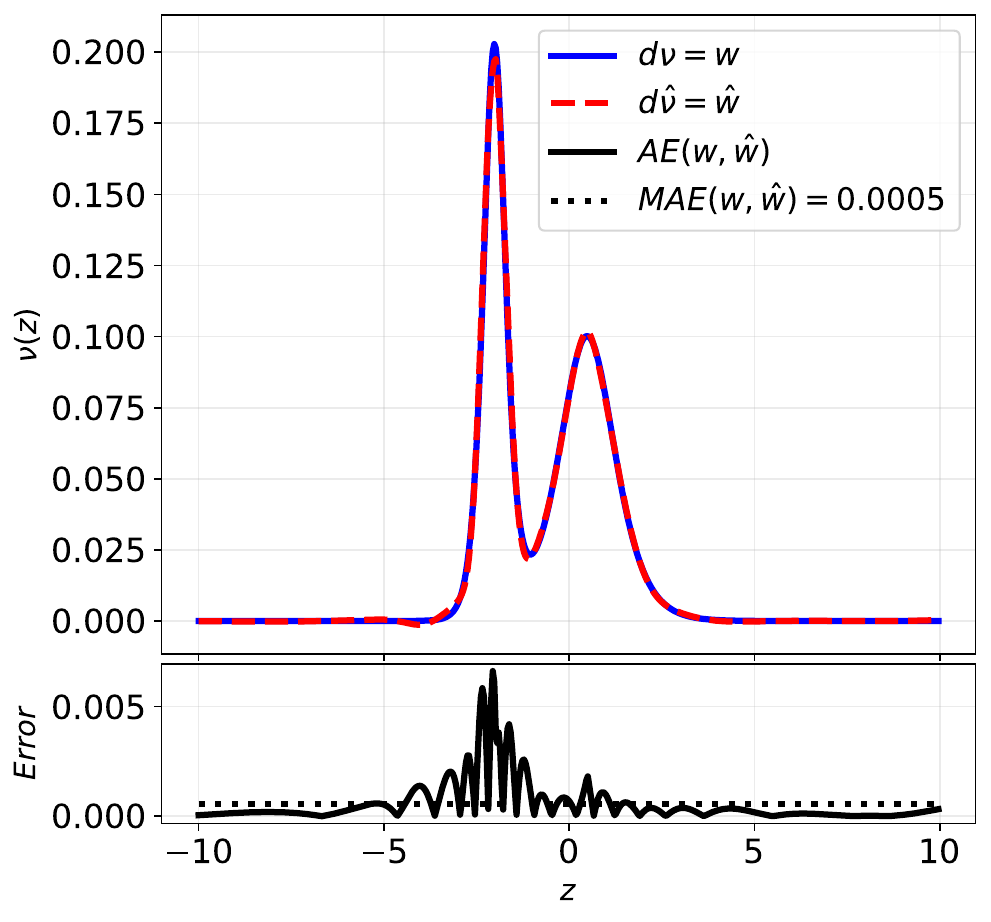} %
    \end{subfigure}
    \caption{Results obtained from Sum of Sigmoid chaining for the synthetic downstream experiment on test set. From left to right: histogram of scores and alignment errors, scatter plot of the aligned and target scores, chaining, and weighting functions, respectively.}\vspace{2em}
    \label{fig:res_synth_reg_short}
\end{figure*}

\subsection{Inventory Optimization as Downstream}\label{sec:realdata_experiment}

The steps for our inventory optimization experiment are similar to~\cref{sec:syn_regression} except that in the third step, we use an actual downstream decision-making task. For this experiment, we adapt the seafood distribution center inventory optimization example from the Pyomo book~\cite{PZGK2025book,PZGK2025online} which follows the Newsvendor model~\cite{khouja1999single,huber2019data}. To apply this optimization task to a real dataset from Kaggle~\cite{kagglerepo}, we made a minor modification to the decision task to operate on a \emph{monthly timescale}. In this setting, a seafood distribution center decides each month on the action of how much tuna to purchase, denoted by $a_t$, at a unit procurement cost $c_t$. After the month ends, customer demand $d_t$ is revealed, and the firm sells up to $\min\{d_t, a_t\}$ tons at a unit selling price $p_t$. Any leftover stock $(a_t - d_t)^+$ incurs a storage holding cost of $h_t$ per ton, where $(\cdot)^+$ is the positive residual (that is $r^+:=\max\{0,r\}$). The expected monthly profit is given by:
\[
\max_{a_t \geq 0} \; \mathbb{E}_{d_t \sim \hat{P}_t} \left[ p_t \min\{d_t, a_t\} - c_t a_t - h_t (a_t - d_t)^+ \right],
\]
where $\hat{P}_t$ denotes the forecast distribution of monthly demand and $t$ is the time index, i.e. $t=\{1,\dots,N\}$.

This formulation slightly differs from the original version in the book as we incorporate \emph{time-varying demand and pricing information} instead of a stationary demand distribution and constant costs. 
For simplicity, we assume that monthly decisions are \emph{independent}, meaning no constraints or inventory carryover link decisions between months. Each month’s demand is modeled independently based on its forecast $\hat{P}_t$. There is \emph{no beginning or ending inventory} — leftover fish is treated as a holding cost rather than carried to the next period. These assumptions allow us to treat each month as a self-contained decision problem. Intuitively, this models a setting where products (e.g., fresh seafood) are perishable or reset monthly, and storage is used temporarily but not for long-term inventory planning.

\begin{table}[ht]
    \centering
    \caption{Evaluation of the Alignment for the Inventory Optimization. Non-aligned evaluation refers to using CRPS while aligned refers to evaluation using our pipeline. Values correspond to Train/Test. MAE is divided by 1e+6, i.e., errors are in terms of million €. 
    }
    \begin{tabular}{cccc}
                    \toprule
                    & $MAE$ & $\tau$ &  $\Delta\tau (\%)$ \\
        \midrule
        Subset 1      &            &              & \\
        Non‐aligned   & 2.54/3.39  & -0.09/-0.22  & -             \\
        Aligned       & 0.72/1.01  &  0.74/0.70   & 83.0/92.0     \\ \hline
        Subset 2      &            &              &  \\
        Non‐aligned   & 0.94/0.97  &  0.25/ 0.51  & -             \\
        Aligned       & 0.25/0.22  &  0.69/ 0.79  & 44.0/24.0     \\ \hline
        Average      &         &              &  \\
        Non‐aligned  & 1.74/2.18 & 0.08/0.15  & -              \\
        Aligned      & 0.49/0.64 & 0.72/0.73  & 63.5/58.0      \\ \bottomrule
        
    \end{tabular}
    
    \label{tab:res_inventory_realdata}
\end{table}

Unlike the experiment from~\cref{sec:syn_regression}, we do not know the true weighting function in this experiment. However, we measure the quality of the alignment based on the metrics introduced in~\cref{sec:alignment_eval_metrics} with an emphasis on Kendall tau, since for alignment, naturally, ranking and correlation of the results matter even more than the scale and magnitude.
\Cref{tab:res_inventory_realdata} shows the result for two subsets of the dataset, showing average performance. Aligned evaluation improves Kendall tau by $58\%$ on average compared to the non-aligned evaluation on the test set. The corresponding alignment curve and the weighting function $\nu$ for "Subset 2" are depicted in~\cref{fig:res_real_downstream}.

\begin{figure}[H]
    \centering
    \begin{subfigure}[b]{0.45\columnwidth}
        \centering
        \includegraphics[width=\textwidth]{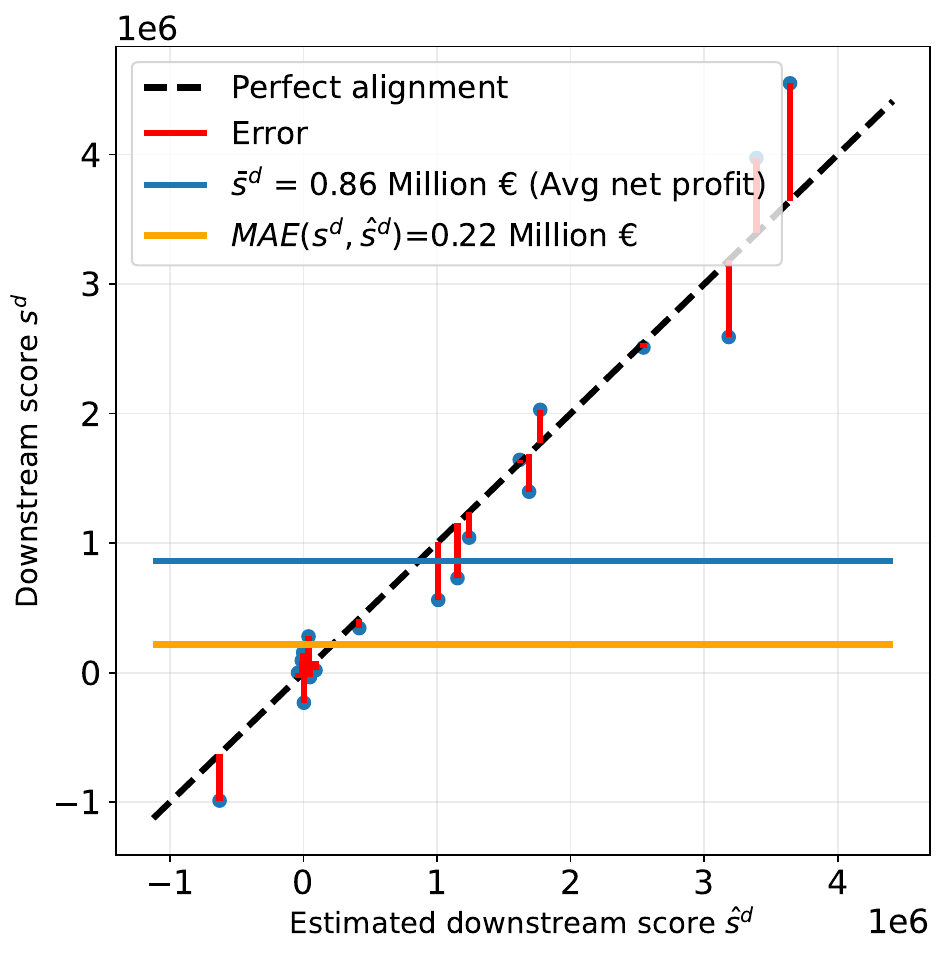} 
    \end{subfigure}
    \begin{subfigure}[b]{0.43\columnwidth}
        \centering
        \includegraphics[width=\textwidth]{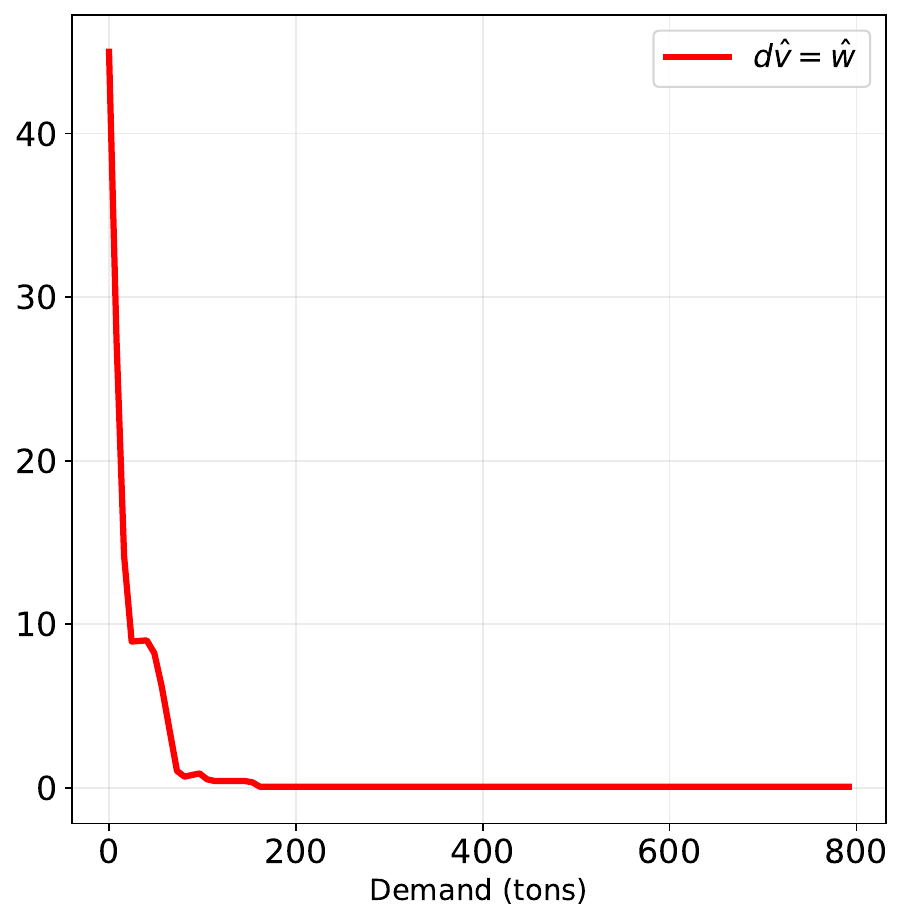} %
    \end{subfigure}
    \caption{Results for "Subset 2" on the test set. \textbf{Left:} alignment curve, \textbf{Right:} learned weighting function $\hat{w}$ corresponding to $\hat{\nu}$.}\vspace{1em}
    \label{fig:res_real_downstream}
\end{figure}
Further details of the experiment can be found in~\cref{sec:appdx_exp_downstream_realdata} of the~\appendixtext. 


\section{Conclusion}\label{sec:conclusion}

In this paper, we have introduced the \emph{evaluation alignment problem}, which arises when standard predictive performance metrics diverge from the true downstream utility of probabilistic forecasts in decision‐making under uncertainty. To address this misalignment, we proposed a novel alignment pipeline that learns a proxy evaluation function directly from downstream scores. By parametrizing a scoring rule with a neural network, our method automatically discovers complex, task‐specific weighting without requiring an expert‐specified weight/cost structure. 
 
Achieving high-quality alignment depends on selecting an expressive function family and having sufficient data coverage. As shown in our synthetic experiments, when these conditions are met, the model can accurately approximate the true alignment function. However, if the training samples are non-representative, the learned alignment may fail to generalize. 
Our results demonstrate that the learned evaluation closely tracks true downstream performance in both of our experiments. 
Alignment performance on real data can improve with more training samples and by including downstream-only variables (e.g., price) in the alignment phase.

In general, the downstream task may be unknown, unobservable, partially observable, or fully observable. When the downstream is unknown, one may rely on standard predictive quality metrics, some of which may still reflect priorities coherent with the downstream application. When it is not fully observable, domain experts can encode preferences through the weighting of scoring rules. In our work, we consider the setting where the downstream task is observable, with sufficient samples to train and validate the alignment model. More complex scenarios may arise, such as aligning to one observable downstream task while transferring alignment to others that are unobservable, an interesting direction for future research. 

Looking ahead, several extensions are promising. First, enriching alignment training with synthetic or bootstrapped samples could better approximate the downstream loss landscape. Second, more complex downstream tasks, including multi-stage or sequential optimization, can be explored. Third, incorporating downstream information such as price signals or context-specific covariates absent at prediction time may refine alignment and capture subtler cost asymmetries. Finally, an especially promising direction is to use the aligned scoring rule as a loss for decision-focused learning. Moreover, beyond univariate regression, the framework also extends naturally to classification and multivariate regression, albeit with added complexities that warrant separate study.



\begin{ack}
This work was supported by the Estonian Research Council grant \href{https://www.etis.ee/Portal/Projects/Display/270ffbc3-9d3b-4be4-a0e8-d3c63ecc8b54}{PRG1604}, by the Estonian Academy of Sciences research professorship in AI, and by the Estonian Centre of Excellence in Artificial Intelligence (\href{https://exai.ee/}{EXAI}), funded by the Estonian Ministry of Education and Research grant \href{https://www.etis.ee/Portal/Projects/Display/97213ae6-cd28-4f84-84c2-ac61f7658670}{TK213}.
\end{ack}


\bibliography{references}

\newpage
\onecolumn
\ifforarxiv
  \section*{Appendix}
\else
  \section*{Supplementary Material}
\fi
\renewcommand{\thesubsection}{A.\arabic{subsection}}
\setcounter{subsection}{0}
\subsection{Proofs}
\textbf{Restatement of Proposition 1:} \\
Let \(L(P,y)\) be a (strictly) proper scoring rule on an outcome space \(\mathcal Y\subseteq\mathbb R\).  Define
\[
S(P,y)\;=\;h\!\bigl(L\bigl(P\circ \nu^{-1},\nu(y)\bigr)\bigr),
\]
where \(\nu:\mathcal Y\to\mathcal Y\) and \(h:\mathbb R\to\mathbb R\) are continuous functions.  Then \(S\) is (strictly) proper if and only if:
\begin{enumerate}
  \item \(\nu\) is a strictly monotonic bijection on \(\mathcal Y\), and
  \item \(h(s)=a\,s+b\) is an affine function with \(a>0\).
\end{enumerate}
No other choice of \((h, \nu)\) preserves propriety.
\begin{proof}
\textbf{Sufficiency.}

By linearity of expectation (affine invariance of the argmin), for any ground truth $Q$ and $h(s)=a s+b$ with $a>0$,
\[
\operatorname*{arg\,min}_{P\in\mathcal P(\mathcal Y)}
\ \mathbb{E}_{Y\sim Q}\!\big[h(L(P,Y))\big]
=
\operatorname*{arg\,min}_{P\in\mathcal P(\mathcal Y)}
\ \big(a\,\mathbb{E}_{Y\sim Q}L(P,Y)+b\big)
=
\operatorname*{arg\,min}_{P\in\mathcal P(\mathcal Y)}
\ \mathbb{E}_{Y\sim Q}L(P,Y),
\]
so the positive affine transformation preserves properness. 

By the change-of-variables formula (pushforward invariance), if $\nu$ is bijective then for any ground truth $Q$,
\[
\mathbb{E}_{Y\sim Q}\,L\big(P\circ\nu^{-1},\nu(Y)\big)
=
\mathbb{E}_{W\sim \nu_{\#}Q}\,L\big(P\circ\nu^{-1},W\big),
\]
where $\nu_{\#}Q$ denotes the pushforward measure $((\nu_{\#}Q)(A)=Q(\nu^{-1}(A)))$.
Since $L$ is (strictly) proper, its conditional risk is minimized (uniquely, in the strict case) at
$P\circ\nu^{-1}=\nu_{\#}Q$, which (because $\nu$ is bijective) is equivalent to $P=Q$.
Hence $S(P,y)=h(L(P\circ\nu^{-1},\nu(y)))$ is (strictly) proper.

\smallskip

\textbf{Necessity.} 

\subparagraph{A.  Non‑monotone $\nu$ breaks convexity.}

Assume $\nu$ is \emph{not} strictly monotone. Then there exist $y_0<y_1$ with either
(i) a \emph{collision} $\nu(y_0)=\nu(y_1)$, or
(ii) an \emph{order reversal} $\nu(y_0)>\nu(y_1)$.

\smallskip
\textit{Collision.}
Let $Q_{1/2}:=\tfrac12\delta_{y_0}+\tfrac12\delta_{y_1}$ be the (equal‐mass) ground truth, and take three forecasts
$P_0:=\delta_{y_0}$, $P_1:=\delta_{y_1}$, and $P_{1/2}:=\tfrac12 P_0+\tfrac12 P_1$.
Because $\nu(y_0)=\nu(y_1)$, we have the same pushforward forecast
\[
P_0\circ\nu^{-1}\;=\;P_1\circ\nu^{-1}\;=\;P_{1/2}\circ\nu^{-1}.
\]
By the change–of–variables identity (pushforward invariance),
\[
\mathbb{E}_{Y\sim Q_{1/2}}\,S(P_k,Y)
\;=\;
\mathbb{E}_{W\sim (\nu_{\#}Q_{1/2})}\,h\!\big(L(P_k\circ\nu^{-1},W)\big)
\qquad (k\in\{0,1,1/2\}),
\]
and the right–hand side is the \emph{same} for $k=0,1,1/2$ because the argument $P_k\circ\nu^{-1}$ is the same.
Thus $P_0$, $P_1$, and $P_{1/2}$ all achieve the same expected score under $Q_{1/2}$, so the Bayes act is not unique.
Hence $S$ cannot be \emph{strictly} proper.

\smallskip
\textit{Reversal.}
If $\nu$ is continuous on $\mathcal Y$ and not strictly monotone, then it is not injective on some interval; in particular,
an order reversal forces a collision at some pair $y_a<y_b$ by the intermediate value theorem.
This reduces to the collision case above, so strict propriety again fails.

\smallskip
In summary, non–bijective $\nu$ (under continuity, non–strictly–monotone) yields multiple Bayes acts for some ground truth,
hence $S$ cannot be strictly proper.

\smallskip
Finally, note that the failure of strict propriety in the collision case
does not depend on the outer transform $h$.  
Whenever two forecasts $P_i$ and $P_j$ yield the same pushforward 
$P_i\circ\nu^{-1}=P_j\circ\nu^{-1}$, then for any outcome $y$ we have
\[
L(P_i\circ\nu^{-1},\nu(y)) \;=\; L(P_j\circ\nu^{-1},\nu(y)),
\]
and consequently
\[
h\!\big(L(P_i\circ\nu^{-1},\nu(y))\big) \;=\;
h\!\big(L(P_j\circ\nu^{-1},\nu(y))\big).
\]
Thus their expected transformed scores coincide under every ground truth,
so $h$ cannot restore uniqueness of the Bayes act.  
In other words, applying $h$ on top of $L$ cannot recover strict propriety
once $\nu$ is non‐injective.

\subparagraph{B.  Non‑affine $h$ breaks convexity.}

Let $L(P,y)$ be a strictly proper scoring loss.  
Define its associated \emph{entropy} as
\[
\Phi(P) \;:=\; \mathbb E_{Y\sim P}[\,L(P,Y)\,].
\]
It is a standard fact \cite{savage1971elicitation} that:
\begin{itemize}
  \item $\Phi$ is convex on the space of probability distributions,
  \item $L$ admits a ``Savage representation'' of the form
  \[
  L(P,y) \;=\; \Phi(P) + D\Phi(P;\,\delta_y - P),
  \]
  where $D\Phi$ is the directional derivative of $\Phi$.
  \item Strict propriety of $L$ is equivalent to strict convexity of $\Phi$.
\end{itemize}

Now, let us consider the scoring rule $S$, which, by assumption, is equal to:
\[
S(P,y) \;=\; h\!\Big(L(P\circ\nu^{-1}, \nu(y))\Big).
\]

Fix a ground truth $Q$. The Bayes risk of $S$ is
\[
\mathcal R_S(P;Q)\;=\;\mathbb E_{Y\sim Q}\,S(P,Y)
\;=\;\mathbb E_{W\sim Q^{\nu}}\!\Big[h\big(L(P^{\nu},W)\big)\Big],
\qquad Q^{\nu}:=\nu_{\#}Q.
\]
For a zero-mass signed direction $H$ (i.e.\ $\int dH=0$, to keep forecast inside probability simplex), consider $P_t=P+tH$ so that $(P_t)^{\nu}=P^{\nu}+t\,\nu_{\#}H$. Using the chain rule and the Savage form,
\[
\frac{d}{dt}\Big|_{t=0}\mathcal R_S(P_t;Q)
=\mathbb E_{W\sim Q^{\nu}}
\!\left[
h'\!\big(L(P^{\nu},W)\big)\cdot D_{Q'}L(Q',W)\big[\,\nu_{\#}H\,\big]\Big|_{Q'=P^{\nu}}
\right].
\]
Since $L(Q',W)=\Phi(Q')+\langle \nabla\Phi(Q'),\,\delta_W-Q'\rangle$, differentiation in $Q'$ (along direction K) gives
\[
D_{Q'}L(Q',W)[K]=\big\langle \mathsf H_\Phi(Q')[K],\,\delta_W-Q'\big\rangle,
\]
hence
\[
D\mathcal R_S(P;Q)[H]\;=\;\Big\langle \,\mathsf H_\Phi(P^{\nu})[\,\nu_{\#}H\,],\;
\mathbb E_{W\sim Q^{\nu}}\!\big[h'\!\big(L(P^{\nu},W)\big)\,(\delta_W-P^{\nu})\big]\Big\rangle,
\]

where $H_\Phi(P^{\nu})$ is a Hessian (Fr\'echet second derivative) of $\Phi$, evaluated at forecast $P^{\nu}$. As shown earlier, the strict propriety of $L$ is equivalent to the strict convexity of $\Phi$.

\medskip
Strict propriety requires that $P=Q$ be a (strict) minimizer of $P\mapsto \mathcal R_S(P;Q)$.
A necessary condition is stationarity for all zero-mass $H$:
\[
0=D\mathcal R_S(Q;Q)[H]
=\Big\langle \,\mathsf H_\Phi(Q^{\nu})[\,\nu_{\#}H\,],\;
M_{Q^{\nu}} \Big\rangle,
\quad\text{where}\quad
M_{Q^{\nu}}:=\mathbb E_{W\sim Q^{\nu}}\!\big[h'\!\big(L(Q^{\nu},W)\big)\,(\delta_W-Q^{\nu})\big].
\]
Assume $\nu$ is injective (because otherwise, as we showed in Part A, strict properness fails) so that $H\mapsto \nu_{\#}H$ spans the full zero-mass tangent space on the $\nu$-pushforward simplex. Since $\mathsf H_\Phi(Q^{\nu})$ is positive definite on that space (as strict properness of $L$ is equivalent to strict convexity of $\Phi$), the only way the inner product above can vanish for \emph{all} $H$ is
\[
M_{Q^{\nu}}=0,\qquad\text{i.e.}\qquad
\mathbb E_{W\sim Q^{\nu}}\!\big[h'\!\big(L(Q^{\nu},W)\big)\,(\delta_W-Q^{\nu})\big]=0.
\tag{$\dagger$}
\]

\medskip
\textbf{Centering implies $h'$ is constant on the range of $L$.}
Identity $(\dagger)$ says that the signed measure with density $h'\!\big(L(Q^{\nu},W)\big)$ against the residual $(\delta_W-Q^{\nu})$ has mean zero.
Equivalently, for every bounded measurable $\varphi$,
\[
\mathbb E\!\left[h'\!\big(L(Q^{\nu},W)\big)\,\varphi(W)\right]
=\mathbb E\!\left[h'\!\big(L(Q^{\nu},W)\big)\right]\cdot \mathbb E\!\left[\varphi(W)\right].
\]
Thus $\operatorname{Cov}\!\left(h'(L(Q^{\nu},W)),\,\varphi(W)\right)=0$ for all $\varphi$, which forces $h'(L(Q^{\nu},W))$ to be $Q^{\nu}$-a.s.\ constant.
If, for some $Q^{\nu}$, the map $w\mapsto L(Q^{\nu},w)$ attains an interval of values (this holds for any strictly proper loss on a binary subproblem), then $h'$ is constant on an interval, hence $h$ is affine on that interval. By connectivity of the attainable range of $L$, $h$ is affine globally on the relevant range:
\[
h(s)=a s + b.
\]

Moreover, if $a<0$, then minimizing $\mathbb E_{Y\sim Q}[S(P,Y)]$ would be equivalent to maximizing $\mathbb E_{Y\sim Q}[L(P,Y)]$, which contradicts propriety.  
If $a=0$, then $S$ is constant and thus not strictly proper.  
Hence we must have $a>0$.

As we showed at the end of Part A, a non-bijective $\nu$ cannot restore strict properness for either an affine or a non-affine $h$.

\end{proof}
\subsection{Algorithms}

Here we provide the algorithms for the predictive modeling step separately for the training in~\cref{alg:predictive_train} and inference step in~\cref{alg:predictive_inference} as well as the downstream optimization in~\cref{alg:decision_optimization} and its evaluation in~\cref{alg:decision_eval}. 
\begin{center}
\begin{minipage}{0.8\linewidth}
\begin{algorithm}[H]
\caption{Predictive Model Training}
\label{alg:predictive_train}
\begin{algorithmic}[1]
\Require Predictive model dataset $\mathcal{D} = \{(\mathbf{x}_i, y_i)\}_{i=1}^N$, probabilistic model $g_\theta$, loss function $l$, learning rate $\eta$, number of epochs $E$, batch size $B$.
\Ensure Trained model $g_{\theta^*}$
\State Shuffle $\mathcal{D}$ and split into $\mathcal{D}_{train}$, $\mathcal{D}_{val}$, $\mathcal{D}_{test}$
\State Initialize model parameters $\theta$

\For{epoch $e = 1$ to $E$}
    \For{each mini-batch $\{(\mathbf{x}_i, y_i)\}_{i \in B} \subset \mathcal{D}_{train}$}
        \State Make predictions
        $\hat{P}_{Y_i} \gets g_\theta(\mathbf{x}_i), \quad \forall i \in B$
        \State Sample from the predictive distribution $\mathbf{\hat{\mathbf{y}}_i} \gets \{\hat{y}^{(j)}_i\}_{j=1}^{M} \sim \hat{P}_{Y_i}$
        \State Compute training loss
        $\mathcal{L}_{train} \gets \frac{1}{B} \sum_{i \in B} l(\hat{\mathbf{y}}_i, y_i)$
        \State Update parameters
        $\theta \gets \theta - \eta \nabla_\theta \mathcal{L}_{train}$
    \EndFor

    \State (Optionally: monitor validation loss 
    $\mathcal{L}_{val} \gets \frac{1}{|\mathcal{D}_{val}|} \sum_{i \in \mathcal{D}_{val}} (\hat{y}_i - y_i)^2$ for early stopping or hyperparameter tuning)
\EndFor
\State \Return $g_{\theta^*}$ (trained model)
\end{algorithmic}
\end{algorithm}
\end{minipage}
\end{center}
\begin{center}
\begin{minipage}{0.8\linewidth}
\begin{algorithm}[H]
\caption{Predictive Model Inference and Evaluation}
\label{alg:predictive_inference}
\begin{algorithmic}[1]
\Require Datasets $\mathcal{D}_{train}, \mathcal{D}_{val}, \mathcal{D}_{test}$, trained model $g_{\theta^*}$
\State Initialize $P_Y^{val} \gets []$, $P_Y^{test} \gets []$
\For{each instance $\mathbf{x}_i \in \mathcal{D}_{val}$}
    \State $\hat{P}_{Y} \gets g_{\theta^*}(\mathbf{x}_i)$
    \State Append $\hat{P}_{Y}$ to $P_Y^{val}$
\EndFor
\For{each instance $\mathbf{x}_i \in \mathcal{D}_{test}$}
    \State $\hat{P}_{Y} \gets g_{\theta^*}(\mathbf{x}_i)$
    \State Append $\hat{P}_{Y}$ to $P_Y^{test}$
\EndFor
\State \Return $P_Y^{val}, P_Y^{test}$
\end{algorithmic}
\end{algorithm}
\end{minipage}
\end{center}

\begin{center}
\begin{minipage}{0.8\linewidth}
\begin{algorithm}[H]
\caption{Decision Model Optimization}
\label{alg:decision_optimization}
\begin{algorithmic}[1]
\Require Downstream objective function $\pi(.,.)$, $\hat{P}_{Y}^{val}, \hat{P}_{Y}^{test}$
 \State $\hat{a}^{val} = \mathop{\arg\min}\limits_{a \in A} \mathbb{E}_{\hat{y} \sim \hat{P}_{Y}^{val}}[\pi(a,\hat{y})]$ \Comment{Solve decision problem on validation set.}
 \State $\hat{a}^{test} = \mathop{\arg\min}\limits_{a \in A} \mathbb{E}_{\hat{y} \sim \hat{P}_{Y}^{test}}[\pi(a,\hat{y})]$ \Comment{Solve decision problem on test set.}
\State \Return optimal decisions $\hat{a}^{val}, \hat{a}^{test}$
\end{algorithmic}
\end{algorithm}
\end{minipage}
\end{center}
\begin{center}
\begin{minipage}{0.8\linewidth}
\begin{algorithm}[H]
\caption{Decision Model Evaluation}
\label{alg:decision_eval}
\begin{algorithmic}[1]
\Require Optimal decisions $\hat{a}^{val}, \hat{a}^{test}$, validation set observations $\mathbf{y}^{val} \in D_{val}$, test set observations $\mathbf{y}^{test} \in D_{test}$
\State Initialize $\mathbf{s}^{d_{val}} \gets [], \mathbf{s}^{d_{test}} \gets []$
\For{each observation $y$ in $\mathbf{y}^{val}$}
    \State $s^d \gets \pi(\hat{a}^{val},y)$
    \State Append $s^d$ to $\mathbf{s}^{d_{val}}$
\EndFor

\For{each observation $y$ in $\mathbf{y}^{test}$}
    \State $s^d \gets \pi(\hat{a}^{test},y)$
    \State Append $s^d$ to $\mathbf{s}^{d_{test}}$
\EndFor

\State \Return $\mathbf{s}^{d_{val}}$, $\mathbf{s}^{d_{test}}$
\end{algorithmic}
\end{algorithm}
\end{minipage}
\end{center}

As the procedures are standard and not novel, they are presented here only for completeness. Moreover, the predictive modeling procedure assumes a neural network-based model, but it could be replaced with any class of probabilistic model. The same goes for the decision task, where it can be replaced with any stochastic programming task. On the other hand the alignment procedure introduced in~\cref{alg:alignment_train} is novel contribution of our work applicable to any predictive task and its downstream counterpart(s).  

The evaluation alignment pipeline requires two sets of inputs. One is the predictions and observations tuple from the predictive modeling phase, and the other is the downstream scores representing the downstream value corresponding to the tuple. Together, they form a triplet of predictions, observations, and downstream scores for the alignment model, where predictions and observations are input to the alignment model while downstream scores are the target.

\subsection{Experiment: Synthetic with Convex Functions}\label{sec:appdx_exp_convex}

In this experiment, which is devised as a sanity check for our alignment pipeline, we employ an alignment model to ensure that it is able to align a simple convex function to another. Illustration in~\cref{fig:illustration_firstpage} is based on this experiment. 

Consider a toy example with two convex functions $f_1(x_i)=2x_i^2$ and $f_2(x_i)=0.5(x_i^4+2x_i^3+2x_i^2)$, where $x_i = \{\hat{y}^{(j)}_i-y_i\}_{j=1}^{M}$ with predictions $\hat{y}^{(j)}_i\sim \mathcal{U}(-3,3)$ and ground truth set to $y_i=0$ for all instances $i$. $f_1$ and $f_2$ are simplified representations of a proper scoring rule for upstream and downstream evaluations, corresponding to $s^u$ and $s^d$, respectively. Alignment aims to map $f_1$ (source) to $f_2$ (target) and have their global optima coincide by transforming $f_1$ while preserving convexity. $\nu$ corresponds to a positive/negative (a)symmetric contraction of the input to $f_1$, which, as a consequence, alters the output of $f_1$ so that it would be aligned with the output of $f_2$. As demonstrated in Fig.~\ref{fig:illustration_2}, an input-scaled $f_1$, denoted as $\tilde{f}_1$ which corresponds to $\hat{s}^d$ estimates $\nu$ denoted as $\hat{\nu}$ to achieve $\tilde{f}_1 = f_2$.

\begin{figure}[ht]
    \centering
    \includegraphics[width=0.5\linewidth]{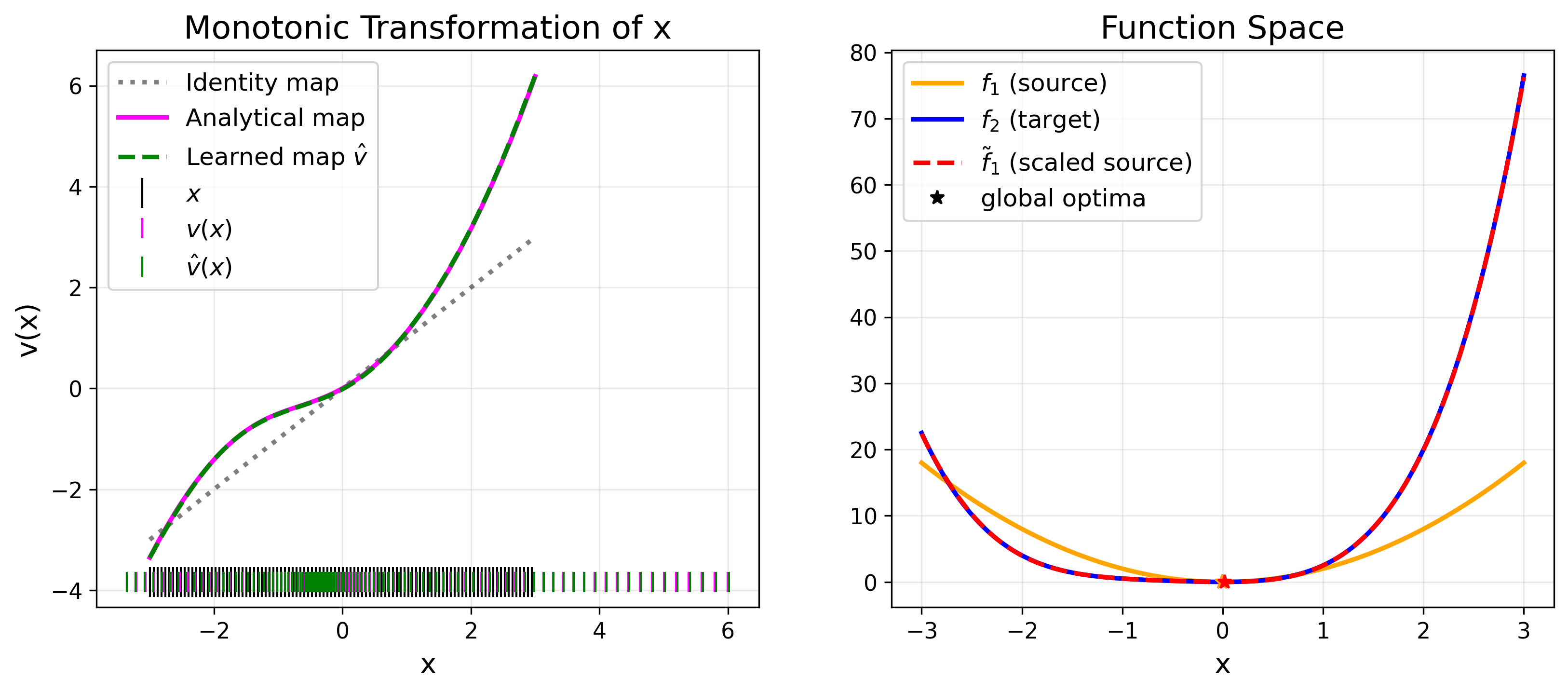}
    \caption{Illustrative toy example. Left: learned monotonic map $\hat{\nu}$ versus the identity map and analytical solution $\nu(x) = \sqrt{\frac{1}{2}f_2}\mathds{1}_{\{x \geq 0\}} - \sqrt{\frac{1}{2}f_2}\mathds{1}_{\{x < 0\}}$. Right: source function $f_1$ and scaled source function $\tilde{f}_1$ matching target function $f_2$.}
    \label{fig:illustration_2}\vspace{3em}
\end{figure}

\subsection{Experiment: Synthetic Downstream (additional details)}\label{sec:appdx_exp_synth_downstream_ext}

Here we provide additional details on the experiment from~\cref{sec:syn_regression} with minor repetitions.
As discussed earlier for this experiment we follow the steps: (1) create a synthetic regression dataset, (2) fit a probabilistic model on this data following the procedure from~\cref{alg:predictive_train}, and then do inference according to~\cref{alg:predictive_inference} to obtain $\hat{P}_Y$, (3) instead of a real downstream decision making task, we adopt a weighted scoring rule, namely, the threshold-weighted CRPS ($twCRPS$) and generate target downstream scores $s^d$ based on a known chaining function $\nu$, (4) use the procedure from~\cref{alg:alignment_train} to train the alignment model to learn the weight function $\nu$ and obtain the estimated downstream score $\hat{s}^d$. 
To the alignment model the weight function is unknown, but this setup allows us to have a ground truth for the weight function which can be seen as yet another sanity check for the alignment pipeline.\\

\subsubsection{Synthetic data generation}
We consider two different functions for synthetic data generating process. Equation~\ref{eq:synth_data_1} corresponds to a sinusoidal periodic function while Eq.~\ref{eq:synth_data_2} corresponds to a quadratic form. Both data generating functions allow for a noise term that could be heterogeneous when $n=1$ or homogeneous when $n\neq1$.

\begin{equation} \label{eq:synth_data_1}
    \begin{aligned}
        f(x, n) &= \frac{1}{2} x + sin(x) + \frac{1}{4} (x\cdot \mathds{1}_{n=1} + \mathds{1}_{n\neq 1}) \varepsilon; \\
        &\quad \text{where}\quad  x \sim \mathcal{U}(-10,10) \quad\text{and}\quad
        \varepsilon \sim \mathcal{N}(0, 1)
    \end{aligned}
\end{equation}

\begin{equation} \label{eq:synth_data_2}
    \begin{aligned}
        f(x, n) &= 2x^2 +0 (\frac{1}{2} x^2\cdot \mathds{1}_{n=1} + \mathds{1}_{n\neq 1}) \varepsilon; \\
        &\quad \text{where}\quad  x \sim \mathcal{U}(0,5) \quad\text{and}\quad
        \varepsilon \sim \mathcal{N}(0, 1)
    \end{aligned}
\end{equation}

Next, we use this data to train a probabilistic regression model.

\subsubsection{Distribution-free Regression Model}
The probabilistic model created here is an arbitrary choice. In principle, it can be replaced with any other probabilistic model. To estimate the conditional distribution of the target variable $y=f(x,n)$ given the input features $x$, we implement a neural network that parameterizes a Gaussian distribution. Specifically, the model learns to predict both the mean and variance of the conditional distribution $\hat{P}_Y = \mathcal{N}(\mu, \sigma^2 | x)$. The network consists of two fully connected hidden layers with ReLU activations:
\begin{equation}
    h_1 = \text{ReLU}(W_1 x + b_1), \quad
    h_2 = \text{ReLU}(W_2 h_1 + b_2).
\end{equation}
From this hidden representation, the model outputs the mean $\mu(x)$ and log-variance $\log \sigma^2(x)$:
\begin{equation}
    \mu(x) = W_{\mu} h_2 + b_{\mu}, \quad 
    \log \sigma^2(x) = W_{\sigma} h_2 + b_{\sigma}.
\end{equation}

The model is trained using the CRPS as the loss function, which measures the quality of probabilistic predictions by comparing the predicted distribution to observed values. Given $N$ training examples $\{(x_i, y_i)\}_{i=1}^N$, the CRPS loss is computed as:
\begin{equation}
    \text{CRPS}(\mu_i, \sigma_i; y_i) = \mathbb{E}_{Y' \sim \mathcal{N}(\mu_i, \sigma_i^2|x)} \left[ |Y' - y_i| - \frac{1}{2} |Y' - Y''| \right],
\end{equation}
where $Y'$ and $Y''$ are independent samples from the predicted distribution. 
In order to obtain samples $Y'$ from the network, the reparameterization trick for Gaussian distribution is used:
\begin{equation}
    Y' = \mu + \sigma\epsilon, \quad \epsilon \sim N(0,1),
\end{equation}
where $\mu = \mu(x)$ and $\sigma = \exp{(\frac{1}{2} \log \sigma^2(x))}$.
Following the procedure from~\cref{alg:predictive_train} we train the model. For training, we use the Adam optimizer with a learning rate of $10^{-2}$ and weight decay of $10^{-4}$. The dataset is split into training (80\%), validation (10\%), and test (10\%) sets. The model is trained for $100$ epochs, and its predictive performance is evaluated on test data following~\cref{alg:predictive_inference} using kernel form of CRPS from~\cref{eq:crps}:

\begin{equation*}
    \begin{aligned}
    \text{CRPS}(\hat{P}_{Y_i}, y_i) &= \E_{Y' \sim \hat{P}_{Y_i}} |Y'_i - y_i| - \frac{1}{2} \E_{Y' \sim \hat{P}_{Y_i}} |Y'_i - Y''_i|,
    \end{aligned}
\end{equation*}

which is estimated as follows:

\begin{equation*}
    \begin{aligned}
        s^u_i = \widehat{CRPS}(\hat{P}_{Y_i}, y_i) = \frac{1}{M}\sum_{j=1}^M |\hat{\mathbf{y}}^{(j)}_i-y_i| - \frac{1}{2M^2} \sum_j\sum_k |\hat{\mathbf{y}}^{(j)}_i - \hat{\mathbf{y}}^{(k)}_i|
    \end{aligned},
\end{equation*}

where $\hat{\mathbf{y}}_i = \{\hat{y}^{(j)}_i \sim Y'\}_{j=1}^{M}$.

\begin{figure}[htbp]
    \centering
    \begin{subfigure}[b]{0.24\textwidth}
        \centering
        \includegraphics[width=\textwidth]{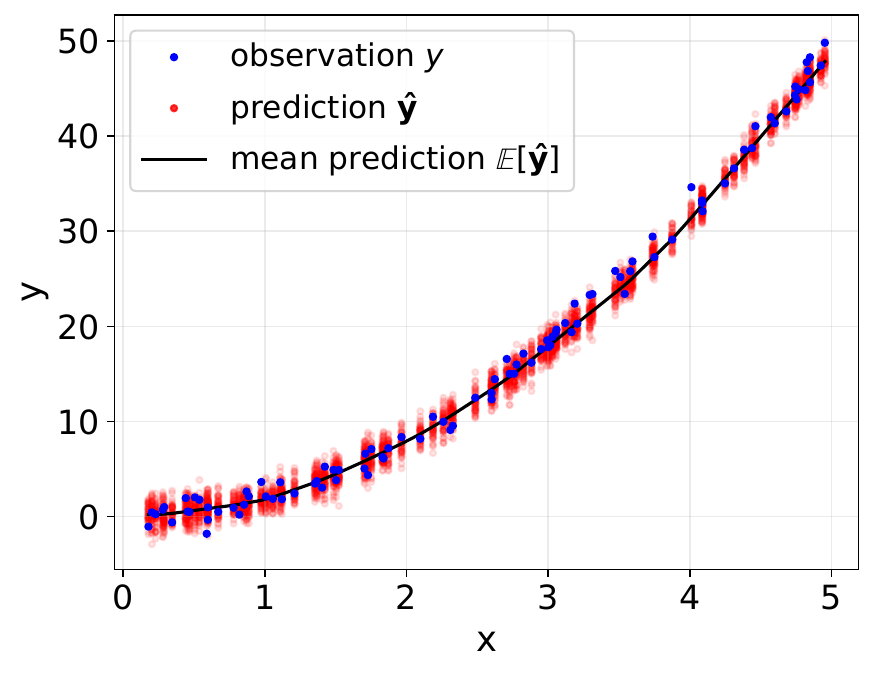} 
        \caption{}
        \label{fig:synth_data_and_pred_sub1}
    \end{subfigure}
    \begin{subfigure}[b]{0.24\textwidth}
        \centering
        \includegraphics[width=\textwidth]{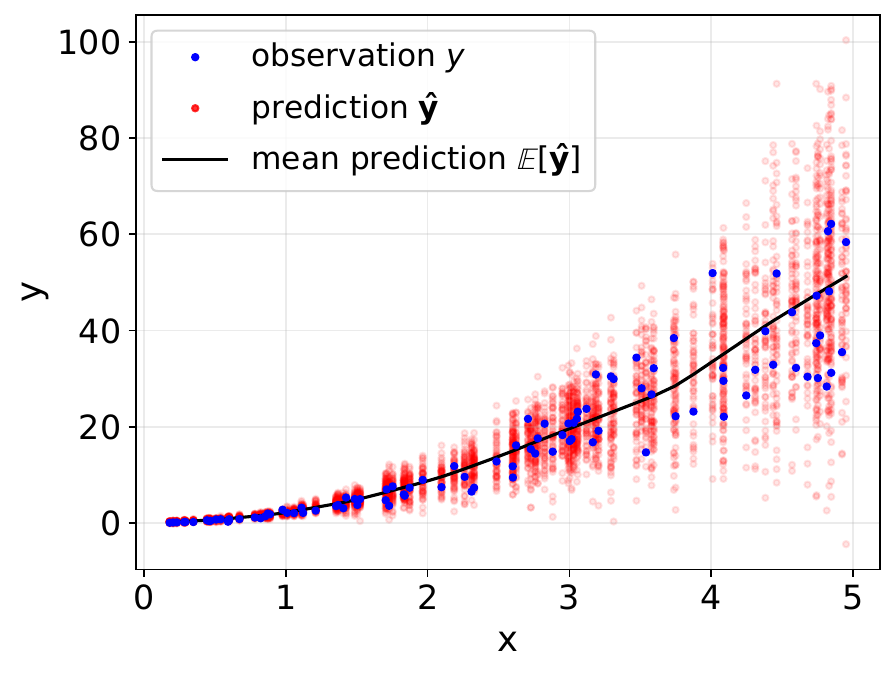} 
        \caption{}
        \label{fig:synth_data_and_pred_sub2}
    \end{subfigure}
    \begin{subfigure}[b]{0.24\textwidth}
        \centering
        \includegraphics[width=\textwidth]{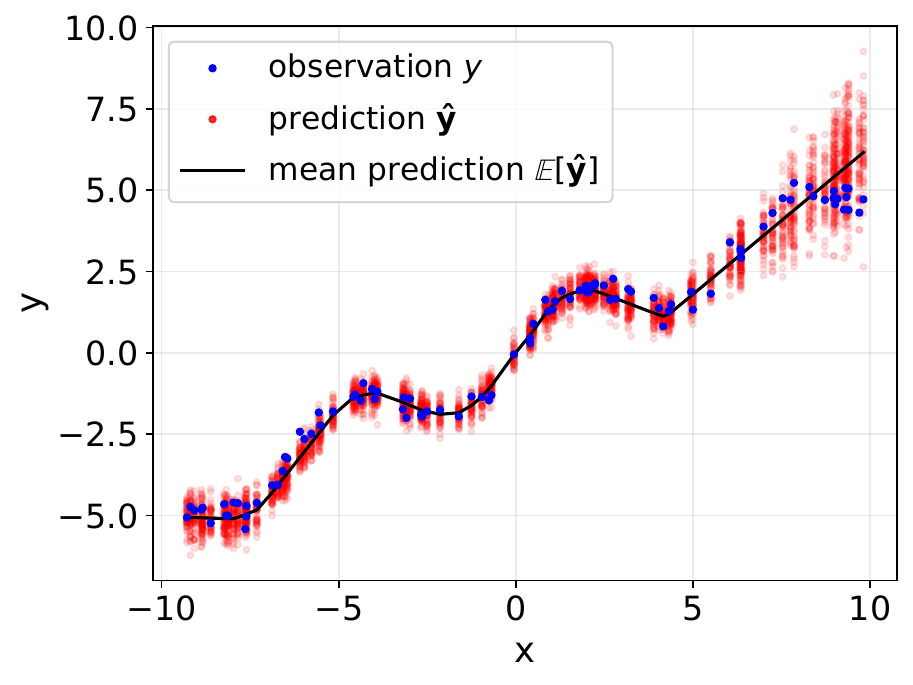} 
        \caption{}
        \label{fig:synth_data_and_pred_sub3}
    \end{subfigure}
    \begin{subfigure}[b]{0.24\textwidth}
        \centering
        \includegraphics[width=\textwidth]{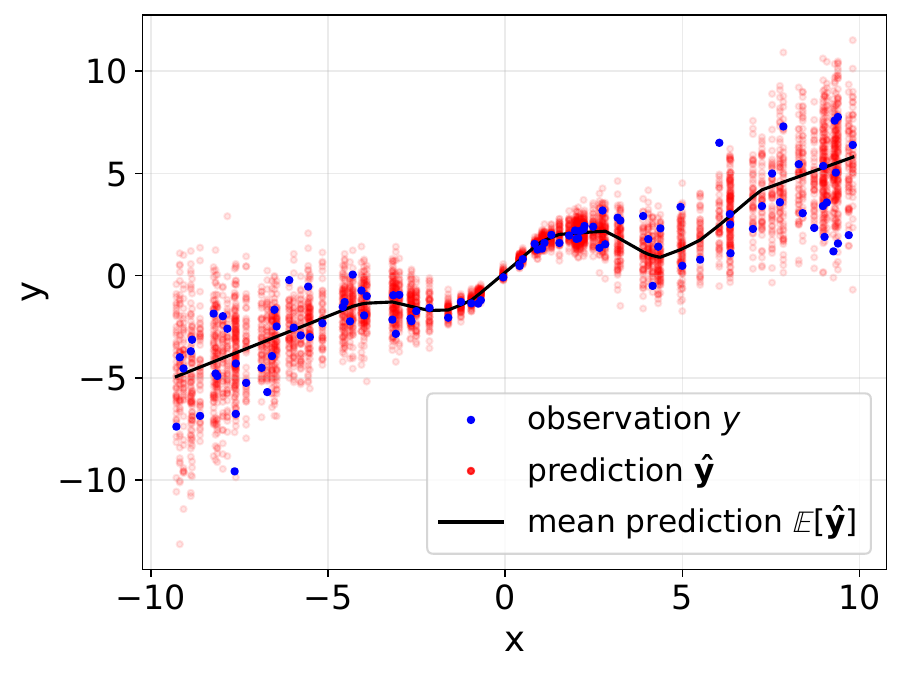}
        \caption{}
        \label{fig:synth_data_and_pred_sub4}
    \end{subfigure}
    \vspace{1em}
    \caption{Synthetic regression data together with the predictions made on test set.}\vspace{2em}
    \label{fig:synth_data_and_pred}
\end{figure}

~\Cref{fig:synth_data_and_pred} depicts the data generated from quadratic and sinusoidal functions from~\cref{eq:synth_data_2,eq:synth_data_1}, respectively.~\cref{fig:synth_data_and_pred_sub1,fig:synth_data_and_pred_sub3} correspond to homogenous noise and~\cref{fig:synth_data_and_pred_sub2,fig:synth_data_and_pred_sub4} to heterogeneous one.

\subsubsection{Synthetic downstream score generation}

Instead of having an actual downstream task, we adopt kernel form of $twCRPS$ from~\cref{eq:twcrps} as the downstream evaluation function and generate synthetic downstream scores with a given chaining function $\nu$: 
\begin{equation*} 
    \begin{aligned}
         twCRPS(\hat{P}_{Y_i}, y_i; \nu) &= \E_{Y' \sim \hat{P}_{Y_i}} |\nu(Y'_i) - \nu(y_i)| - \frac{1}{2} \E_{Y' \sim \hat{P}_{Y_i}} |\nu(Y'_i) - \nu(Y''_i)|,
    \end{aligned}
\end{equation*}

which is estimated as follows,

\begin{equation*}
    \begin{aligned}
        s^d_i &= \widehat{twCRPS}(\hat{P}_{Y_i}, y_i; \nu) \\
              &= \frac{1}{M}\sum_{j=1}^M |\nu(\hat{\mathbf{y}}^{(j)}_i)-\nu(y_i)| - \frac{1}{2M^2} \sum_j\sum_k |\nu(\hat{\mathbf{y}}^{(j)}_i) - \nu(\hat{\mathbf{y}}^{(k)}_i)|,
    \end{aligned}
\end{equation*}

where $\hat{\mathbf{y}}_i = \{\hat{y}^{(j)}_i \sim Y'\}_{j=1}^{M}$.

For chaining function $\nu$, we consider different functions as reported in~\cref{tab:chaining_table} (depicted below again for easier reference). Chaining functions I-III are based on~\cite{allen2023evaluating} while we devised chaining function IV as a more complex case for the chaining function. \\

\begin{table*}[ht]
    \centering
    \begin{tabular}{ccl}
         & Name & Chaining Function \\\hline
        1 & Threshold & $\nu(z;t) = max(z, t)$ \\
        2 & Interval & $\nu(z;a,b) = min(max(z, a), b); a < b$ \\
        3 & Gaussian & $\nu(z;t,\mu,\sigma) = (z-t)  \Phi_{\mu, \sigma}(z) + \sigma^2 \phi_{\mu, \sigma}(z)$ \\
        4 & SumSigmoids & $\nu(z;\mathbf{a},\mathbf{b},\mathbf{c},\mathbf{d}) = \sum_i \frac{c_i}{1+\exp{(-(a_iz + b_i))}+d_i}$ \\
    \end{tabular}
\end{table*}

Examples of each chaining function type is illustrated in~\cref{fig:synth_chaining}. A threshold function with $t=0.5$, an interval function with $a=-0.5, b=1.5$, a Guassian function with $t=0.5, \mu=0, \sigma=1$, and Sum of Sigmoid with $\mathbf{a}=(0,5,1,2), \mathbf{b}=(2,10,20,-1), \mathbf{c}=(1,4,2,5), \mathbf{d}=(0,0,0,0)$ are depicted in~\crefrange{fig:synth_chaining_sub1}{fig:synth_chaining_sub4}, respectively.

\begin{figure}[htbp]
    \centering
    \begin{subfigure}[b]{0.24\textwidth}
        \centering
        \includegraphics[width=\textwidth]{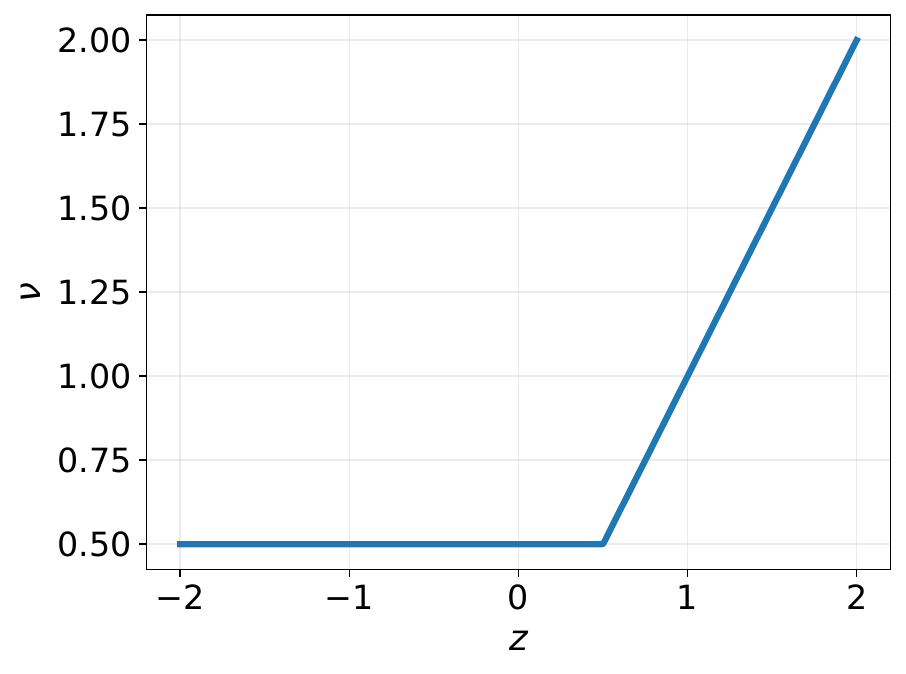} 
        \caption{Threshold}
        \label{fig:synth_chaining_sub1}
    \end{subfigure}
    \begin{subfigure}[b]{0.24\textwidth}
        \centering
        \includegraphics[width=\textwidth]{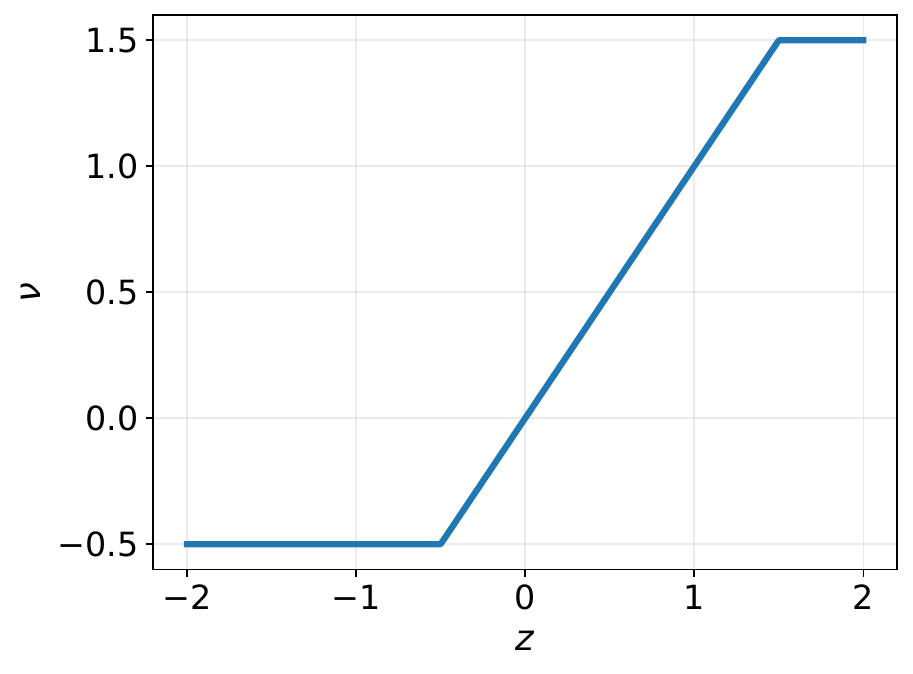} %
        \caption{Interval}
        \label{fig:synth_chaining_sub2}
    \end{subfigure}
    \begin{subfigure}[b]{0.24\textwidth}
        \centering
        \includegraphics[width=\textwidth]{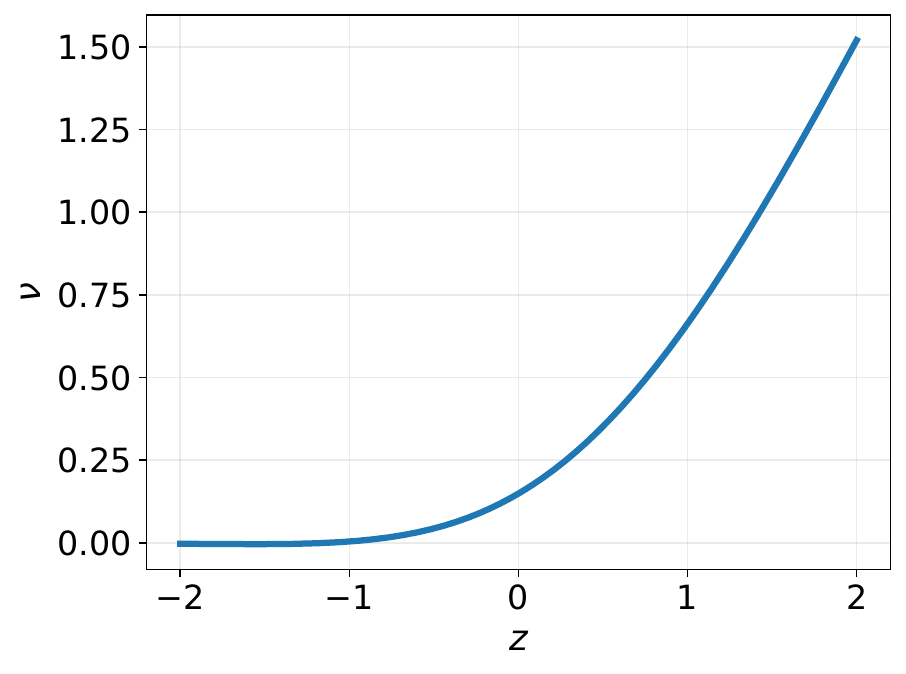} %
        \caption{Gaussian}
        \label{fig:synth_chaining_sub3}
    \end{subfigure}
    \begin{subfigure}[b]{0.24\textwidth}
        \centering
        \includegraphics[width=0.97\textwidth]{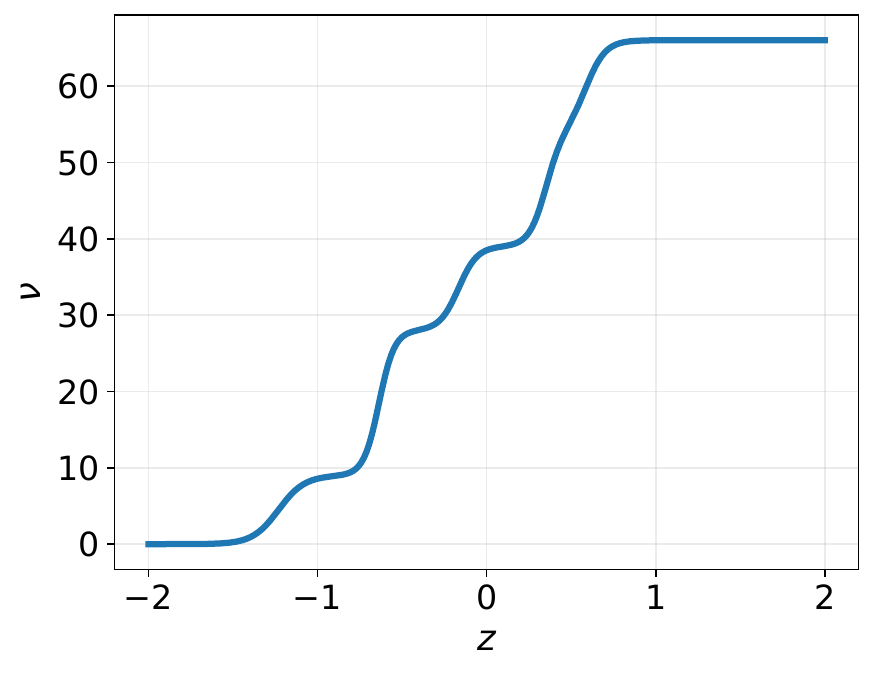}
        \caption{SumSigmoids}
        \label{fig:synth_chaining_sub4}
    \end{subfigure}
    \vspace{1em}
    \caption{Chaining functions used to generate synthetic downstream scores $s^d$}\vspace{2em}
    \label{fig:synth_chaining}
\end{figure}

\subsubsection{Alignment and its evaluation}

We use the parametrized kernel form of twCRPS from~\cref{eq:twcrps} using a neural network $f$ with parameters $\bm{\theta}$ according to~\cref{eq:parameterized_tw_scoringrule}:

\begin{equation*}
    \begin{aligned}
        \hat{s}^d_i &= f(\hat{\mathbf{y}}_i, y_i, twCRPS; \bm{\theta}) \\
                    &= h(twCRPS(\hat{\mathbf{y}}_i, y_i; \nu(.;\theta_1)); \theta_2)
    \end{aligned}
\end{equation*}

which is estimated as follows,
\begin{equation*}
    \begin{aligned}
        \hat{s}^d_i &= h( \Big[\frac{1}{M}\sum_{j=1}^M |\nu(\hat{\mathbf{y}}^{(j)}_i, \theta_1)-\nu(y_i, \theta_1)| - \frac{1}{2M^2} \sum_j\sum_k |\nu(\hat{\mathbf{y}}^{(j)}_i, \theta_1) - \nu(\hat{\mathbf{y}}^{(k)}_i, \theta_1)|\Big]; \theta_2).
    \end{aligned}
\end{equation*}

Where $h(x;\theta_2) = wx+b; w > 0$. We follow~\cref{alg:alignment_train} to obtain $\hat{s}^d_i$. The results of the alignment based on sinusoidal data with heterogeneous noise for different chaining functions on the test set are shown in~\cref{fig:res_synth_reg}.

\subsubsection{Hyperparameters}\label{sec:hyperparams}

In our work, we primarily considered CRPS as the upstream evaluation metric. There are relatively few widely used (strictly) proper scoring rules for univariate regression. We chose threshold-weighted CRPS for its broad applicability and compatibility with sample-based, non-parametric distributional predictions. Alternatives like pinball loss and interval scores apply only to specific quantiles and, when averaged over many quantiles, effectively approximate CRPS. The log score requires density access, which is often impractical for sample-based settings. The recently proposed Inverse Multiquadric Score (IMS)~\cite{allen2023evaluating} is a promising option but is not yet widely adopted.

While we did not conduct an automated hyperparameter search, we empirically explored different architectural variations. The architecture was designed on synthetic datasets and then applied to the real data experiments while ensuring the test set remained strictly held out. Specifically, we tested up to three monotonic layers, varying neuron counts up to 100, weight decay with varying strength, and activation functions (e.g., ReLU, GELU). Ultimately, as depicted in Fig.~\ref{fig:architecture}, input-transformation $\nu$ with one hidden layer of size $n=50$ and output-transformation $h$ with two layers of size $n=10$, weight decay of $1e-5$, learning rate of $0.04$, and epoch of 100 proved to be sufficient across all cases. ReLu was chosen as activation function across all experiments except for the synthetic experiment with Sum of Sigmoids chaining function where it was GELU as depicted on the last row of~\cref{fig:res_synth_reg}.  

\begin{figure}[htbp]
    \centering
    \begin{subfigure}[b]{0.23\textwidth}
        \centering
        \includegraphics[width=\textwidth]{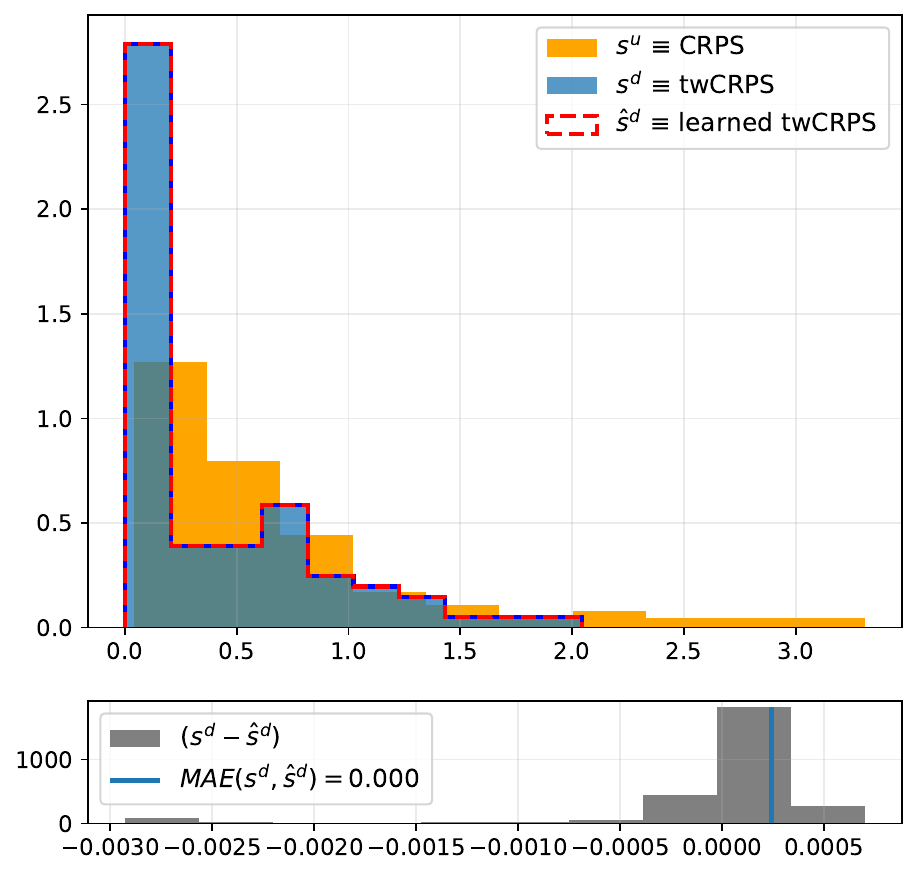} 
    \end{subfigure}
    \begin{subfigure}[b]{0.225\textwidth}
        \centering
        \includegraphics[width=\textwidth]{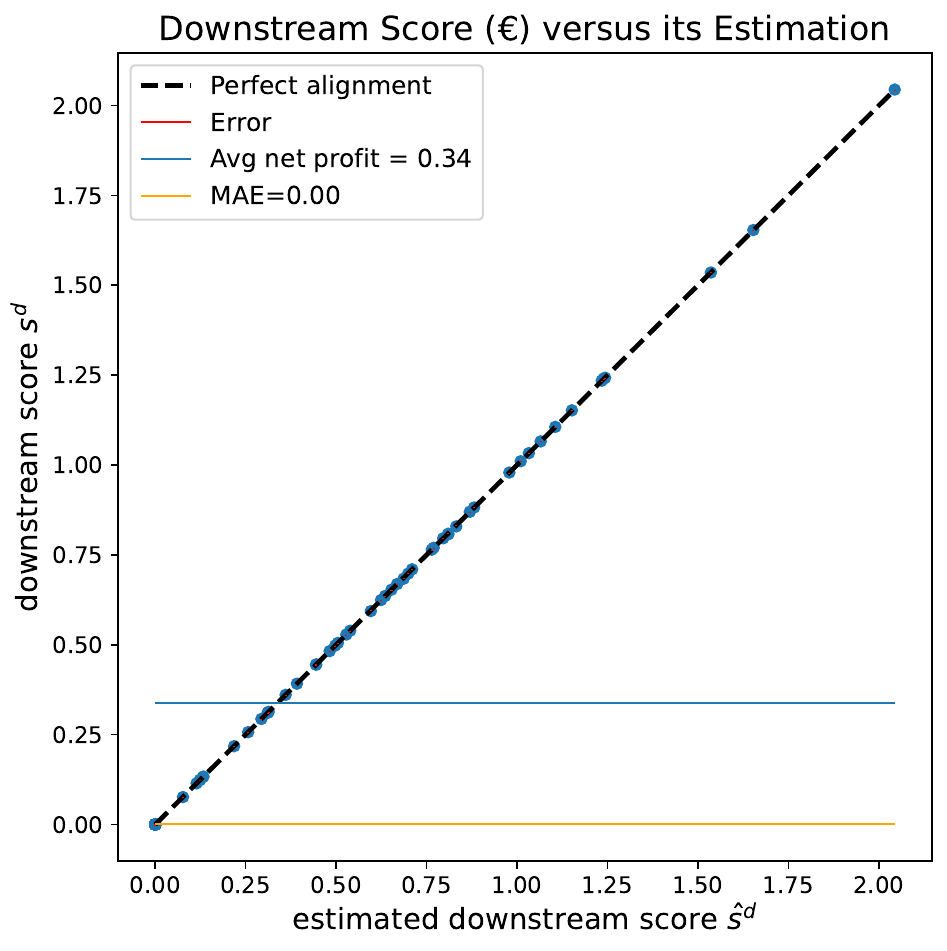} 
    \end{subfigure}
    \begin{subfigure}[b]{0.23\textwidth}
        \centering
        \includegraphics[width=\textwidth]{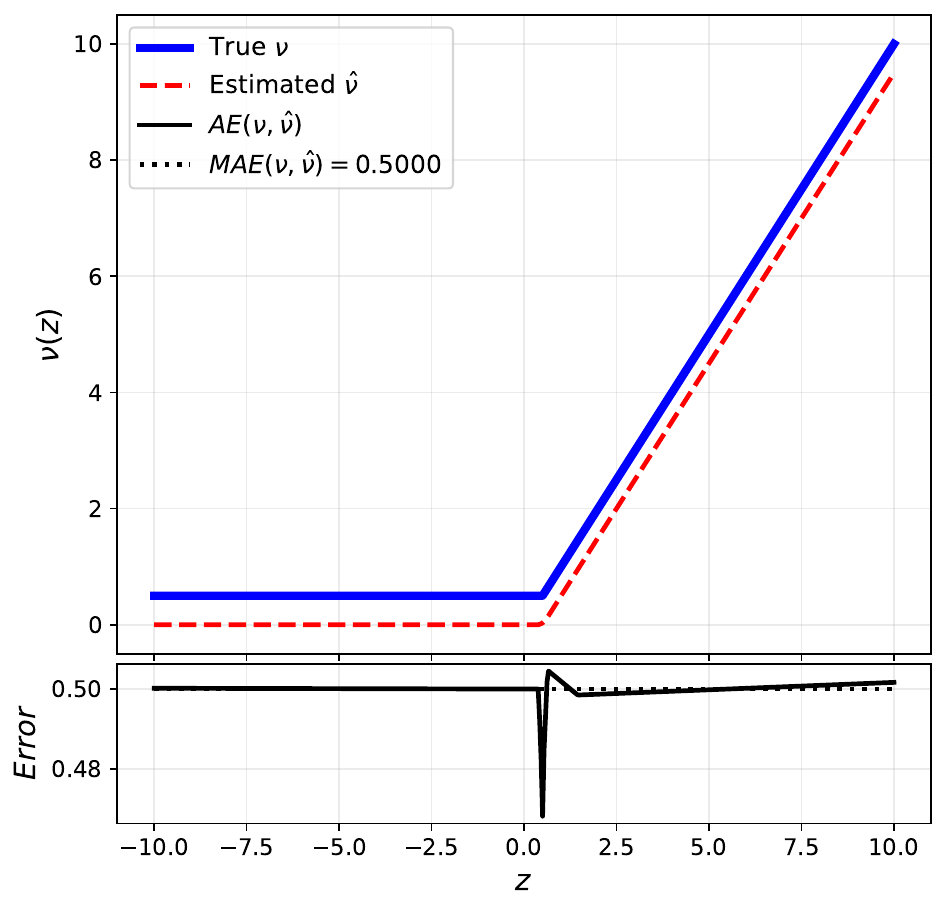} %
    \end{subfigure}
    \begin{subfigure}[b]{0.23\textwidth}
        \centering
        \includegraphics[width=\textwidth]{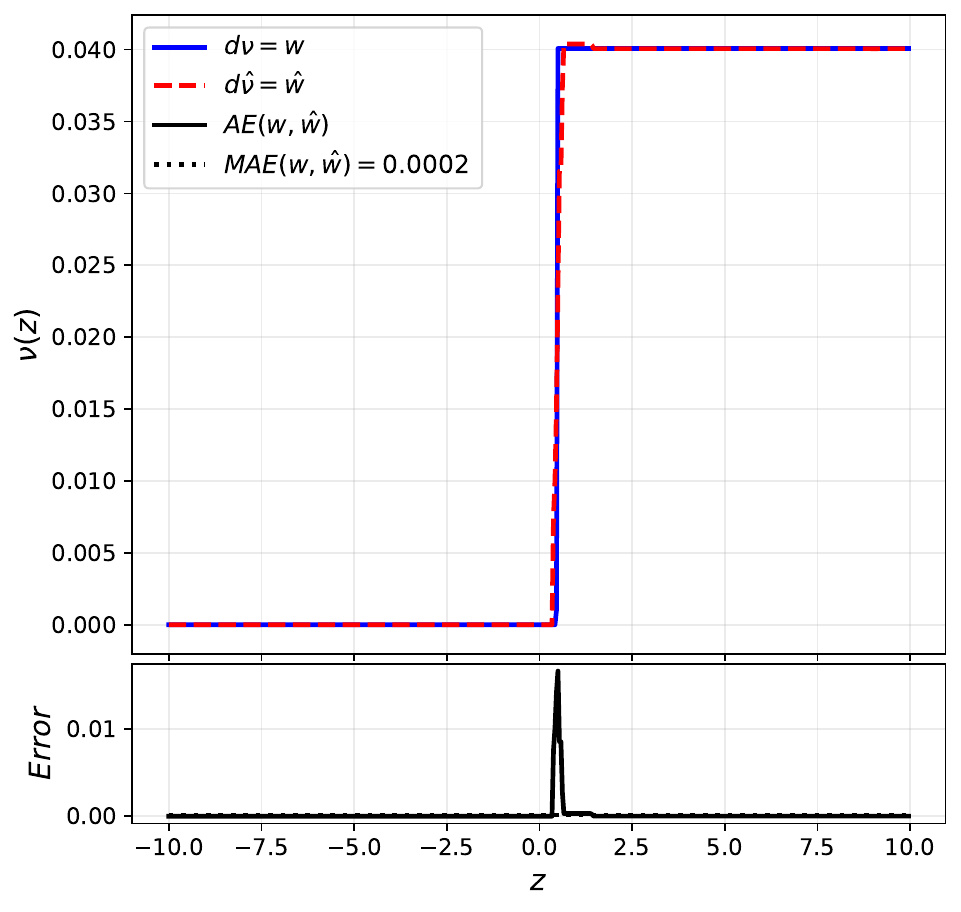} %
    \end{subfigure}

    \hfill

    \begin{subfigure}[b]{0.23\textwidth}
        \centering
        \includegraphics[width=\textwidth]{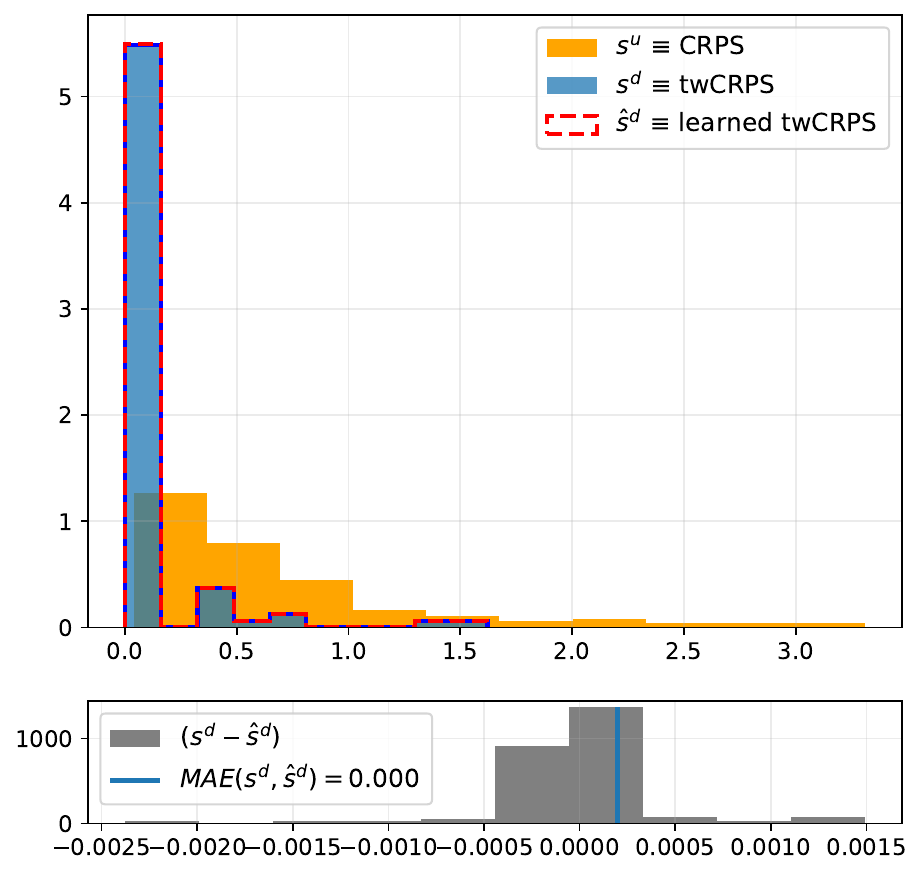} 
    \end{subfigure}
    \begin{subfigure}[b]{0.225\textwidth}
        \centering
        \includegraphics[width=\textwidth]{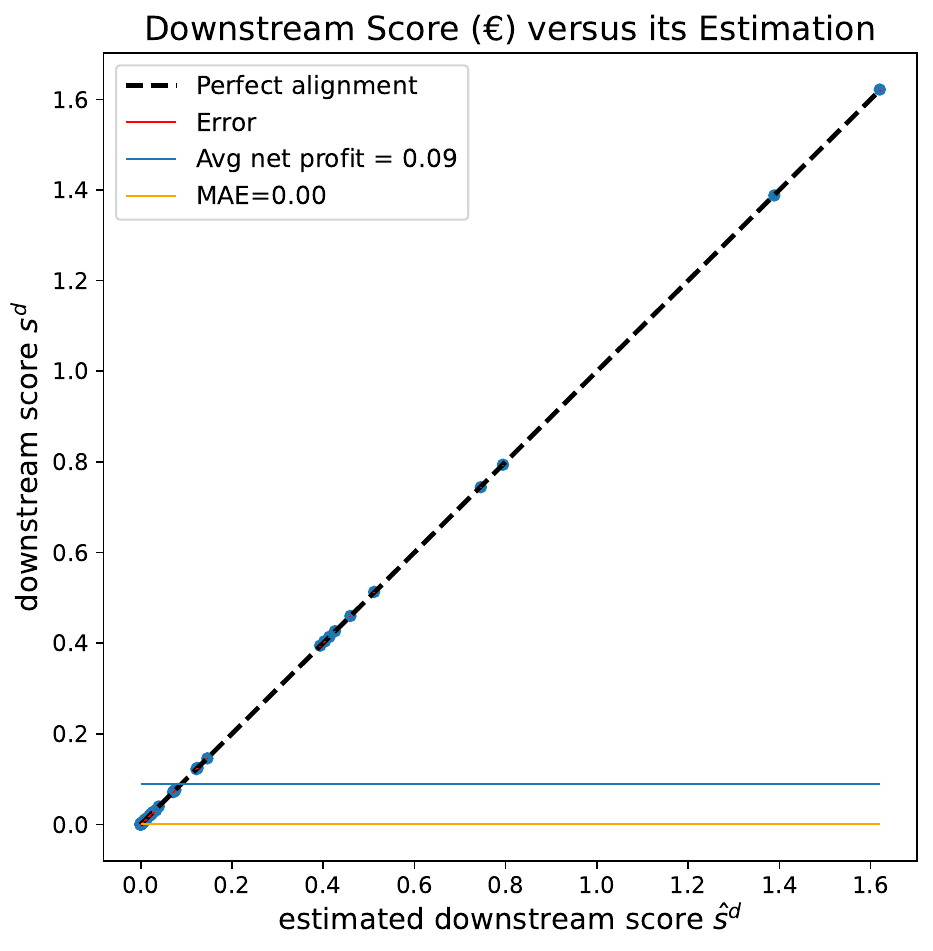} 
    \end{subfigure}
    \begin{subfigure}[b]{0.23\textwidth}
        \centering
        \includegraphics[width=\textwidth]{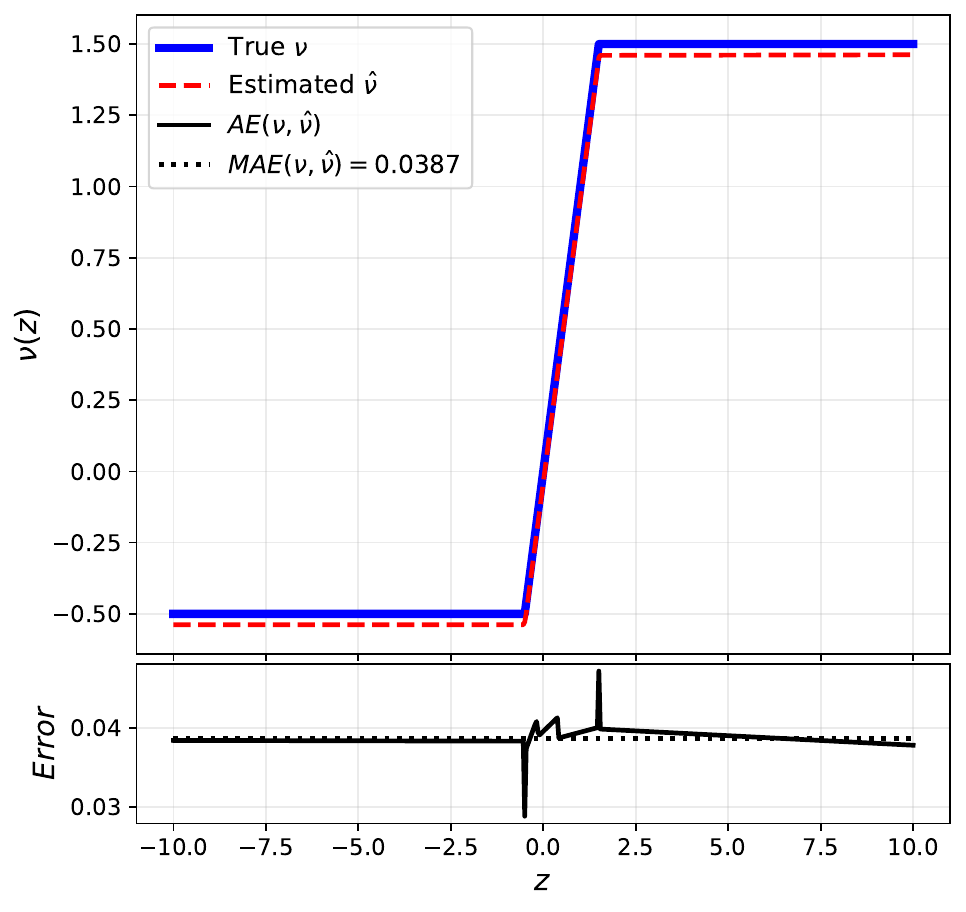} %
    \end{subfigure}
    \begin{subfigure}[b]{0.23\textwidth}
        \centering
        \includegraphics[width=\textwidth]{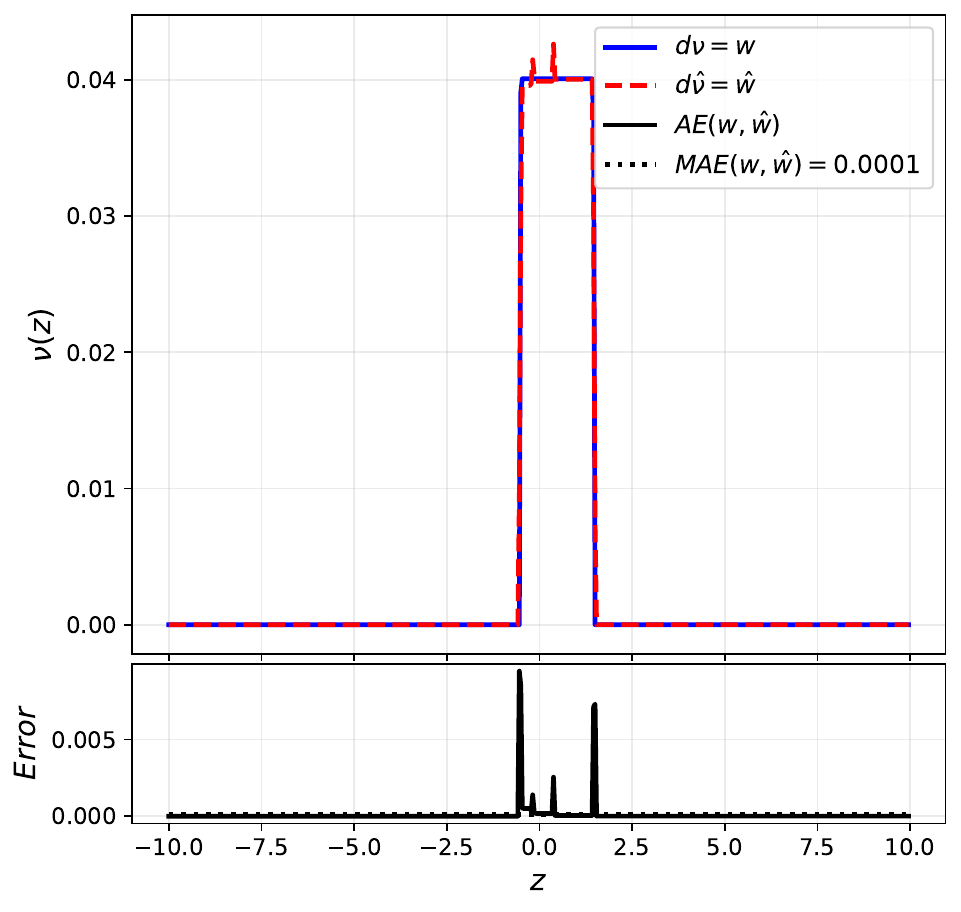} %
    \end{subfigure}

    \hfill

    \begin{subfigure}[b]{0.23\textwidth}
        \centering
        \includegraphics[width=\textwidth]{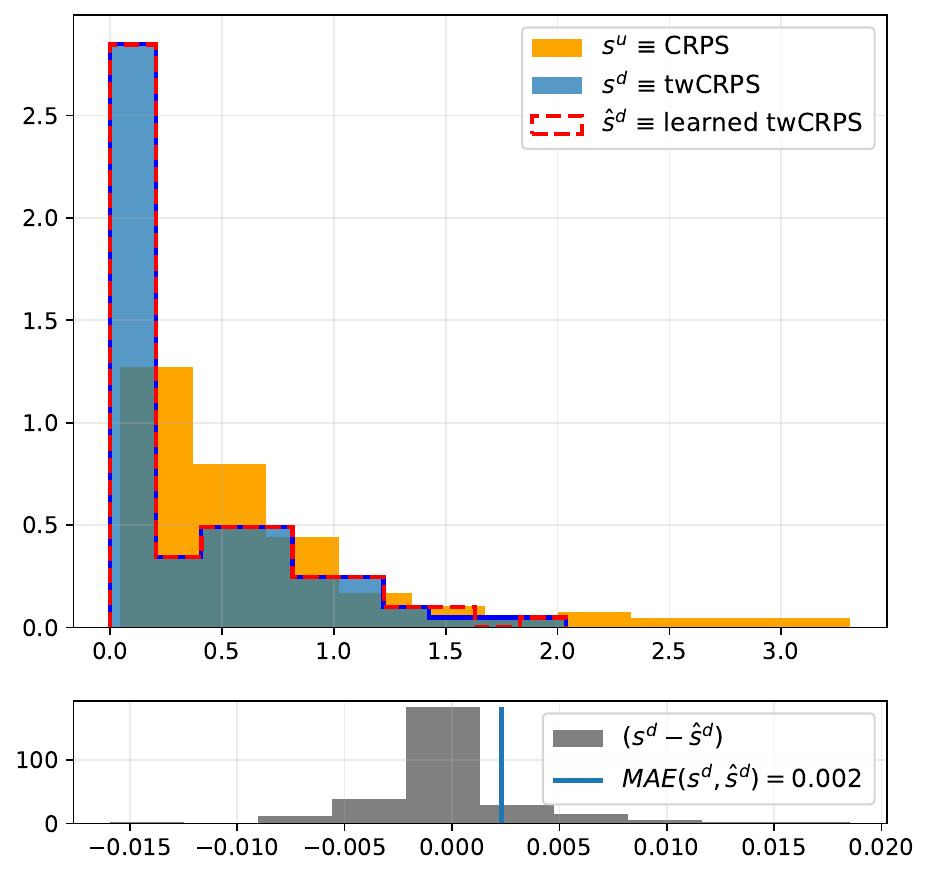} 
    \end{subfigure}
        \begin{subfigure}[b]{0.225\textwidth}
        \centering
        \includegraphics[width=\textwidth]{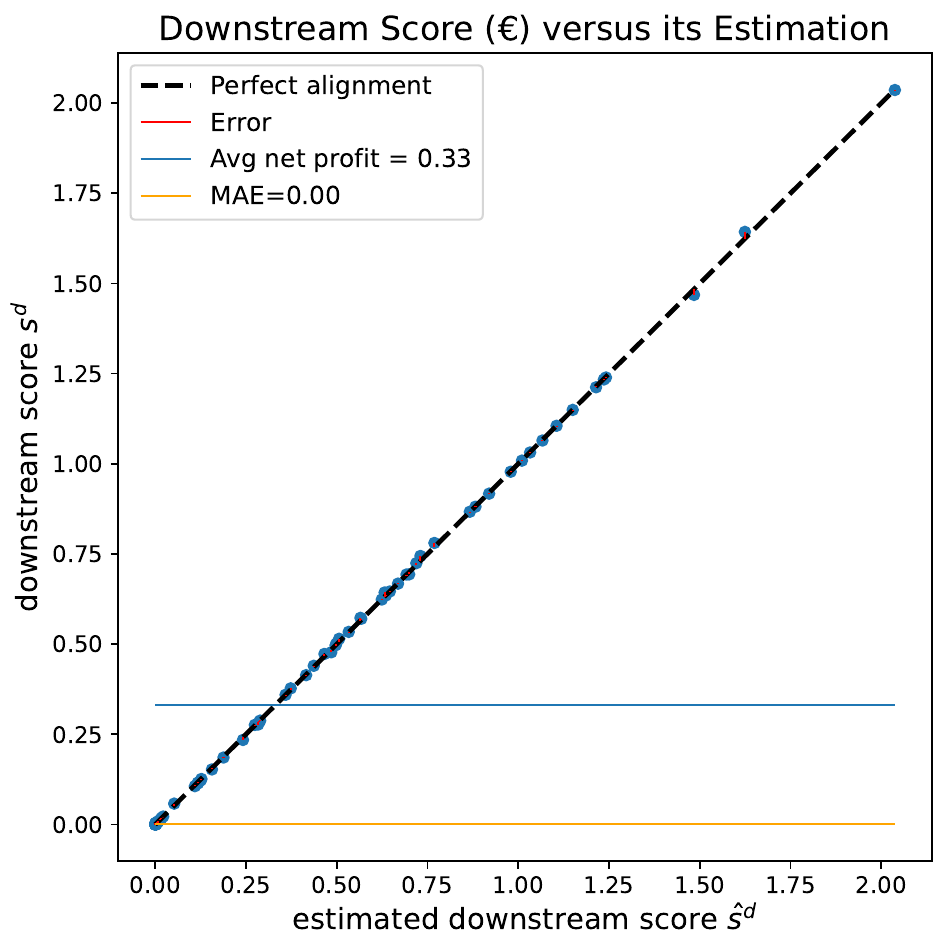} 
    \end{subfigure}
    \begin{subfigure}[b]{0.23\textwidth}
        \centering
        \includegraphics[width=\textwidth]{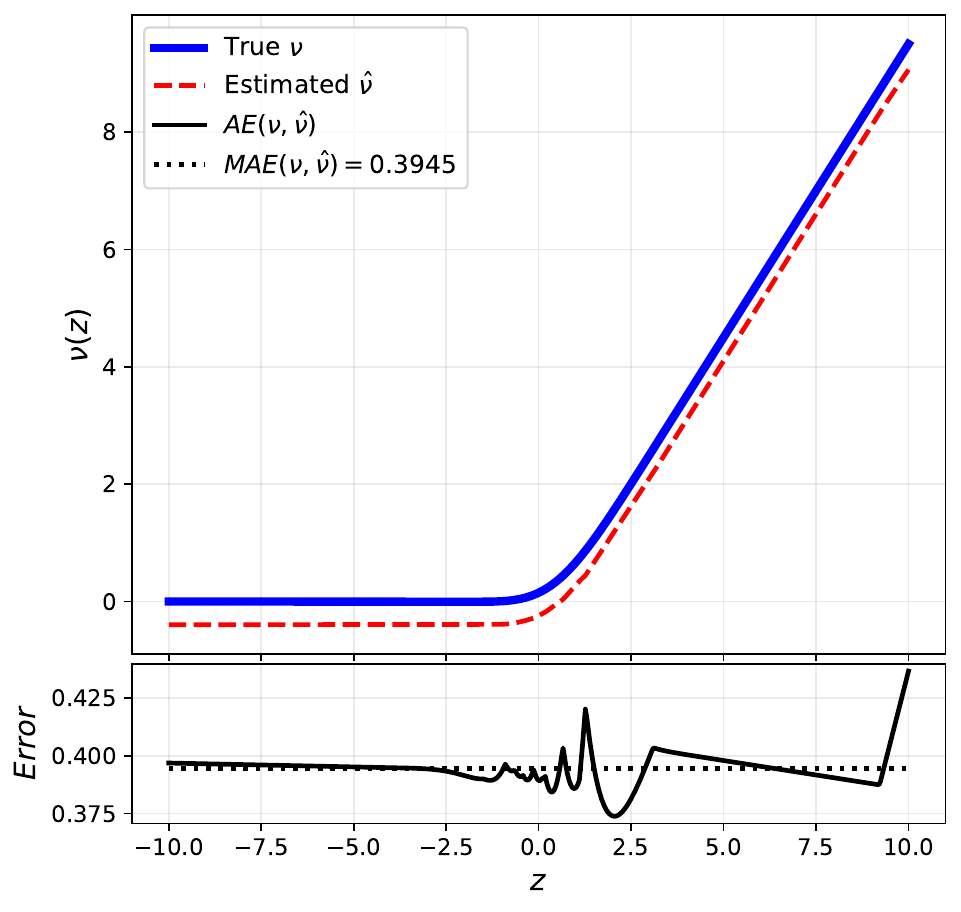} %
    \end{subfigure}
    \begin{subfigure}[b]{0.23\textwidth}
        \centering
        \includegraphics[width=\textwidth]{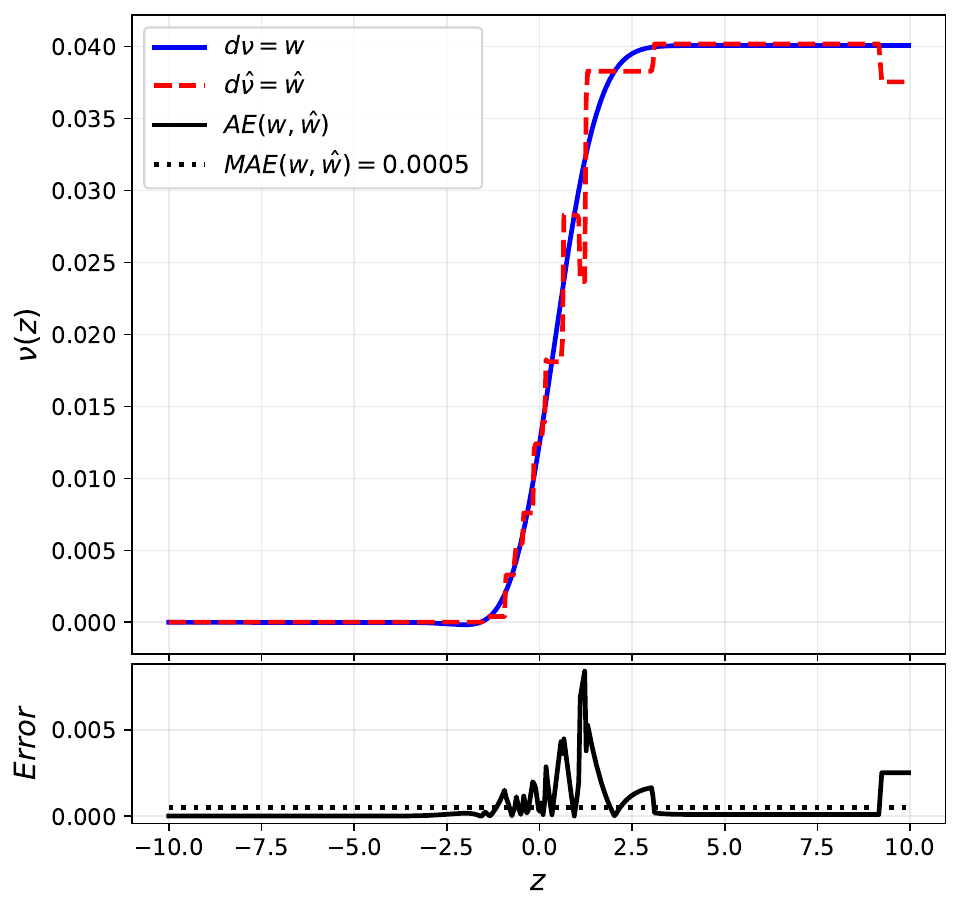} %
    \end{subfigure}

    \hfill

    \begin{subfigure}[b]{0.23\textwidth}
        \centering
        \includegraphics[width=\textwidth]{./data-sinusoidal_noise-hetero_result_scores_and_error_distributions_ctype_sumsigmoids_test.pdf} 
        \caption{}
    \end{subfigure}
    \begin{subfigure}[b]{0.225\textwidth}
        \centering
        \includegraphics[width=\textwidth]{./data-sinusoidal_noise-hetero_result_scores_alignment_curve_ctype_sumsigmoids_test.pdf} 
        \caption{}
    \end{subfigure}
    \begin{subfigure}[b]{0.23\textwidth}
        \centering
        \includegraphics[width=\textwidth]{./data-sinusoidal_noise-hetero_result_chaining_and_error_ctype_sumsigmoids_test.pdf} %
        \caption{}
    \end{subfigure}
    \begin{subfigure}[b]{0.23\textwidth}
        \centering
        \includegraphics[width=\textwidth]{./data-sinusoidal_noise-hetero_result_weighting_and_error_ctype_sumsigmoids_test.pdf} %
        \caption{}
    \end{subfigure}
    \vspace{1em}
    \caption{Results obtained from threshold, interval, Gaussian, and SumofSigmoids chaining from top to bottom, respectively. (a) distribution of scores, (b) alignment curves, (c,d) learned chaining and weighting functions with their corresponding error to the true functions.}
    \label{fig:res_synth_reg}
\end{figure}

\newpage
\subsection{Experiment: Inventory Optimization with Real Data (additional details)}\label{sec:appdx_exp_downstream_realdata}

The dataset consists of 168 months, of which we dedicate 144 months to training and 24 months to the test range as depicted in~\cref{fig:realdata_train_test_split} Due to the time-series nature of the dataset, we obtain the prediction for the validation set by doing backtesting over the train data with a 1-month forecast horizon and strides of length one. This provides us with 120 months within the validation set. Thus, the alignment set consists of 120 instances, and the test set for alignment consists of 24 months. \Cref{fig:downstream_realdata_subset2_example} depicts the predictions made on the validation set. The probabilistic demand predictions are issued using the Exponential Smoothing model from the Darts library~\cite{JMLR:v23:21-1177}. The alignment model architecture and hyperparameters were the same as described under~\cref{sec:hyperparams} except for the number of epochs that was set to 1000 for this experiment. 

\begin{figure}[H]
    \centering
    \includegraphics[width=0.75\linewidth]{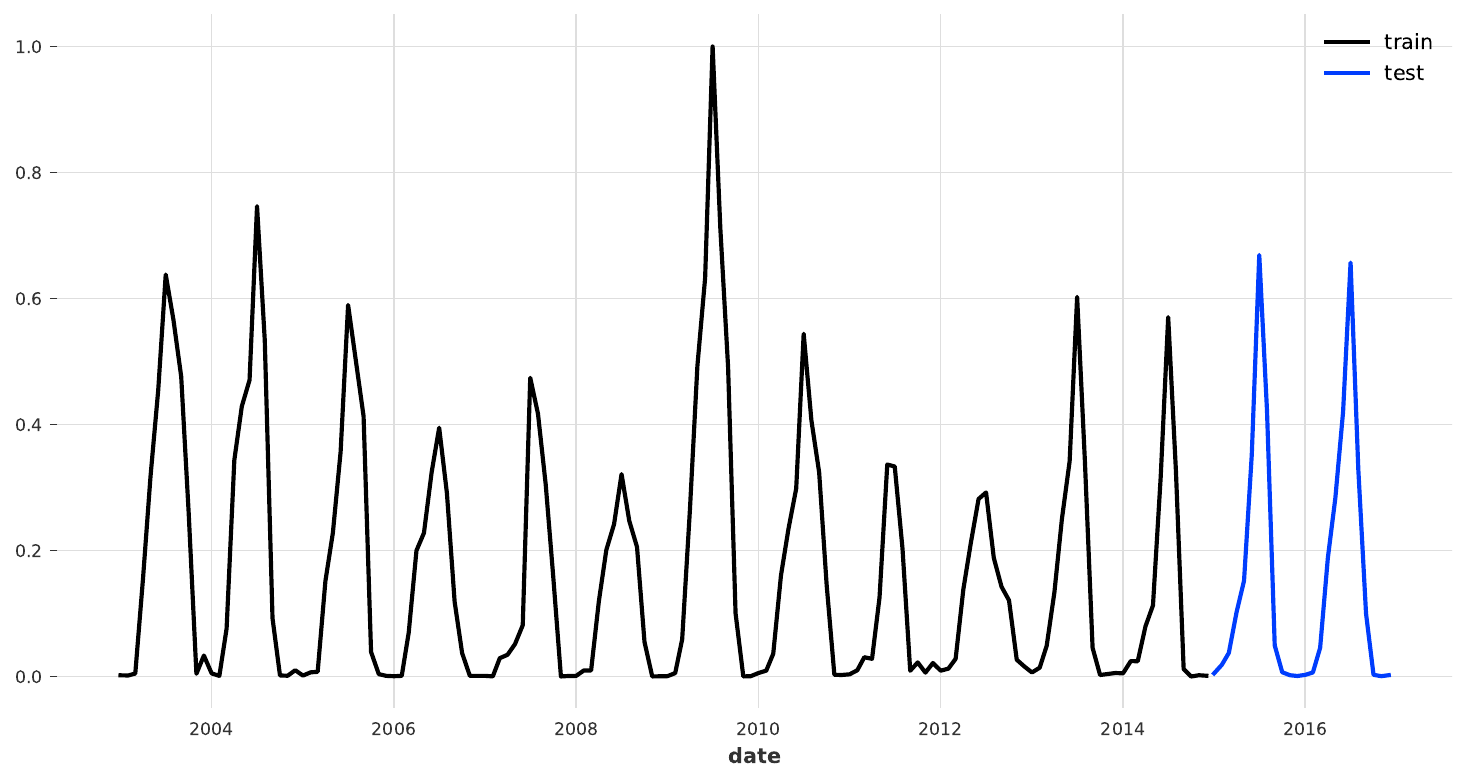}
    \caption{Example of real demand data used for building a probabilistic model.}
    \label{fig:realdata_train_test_split}
\end{figure}

\begin{figure}[H]
    \centering
    \includegraphics[width=0.75\linewidth]{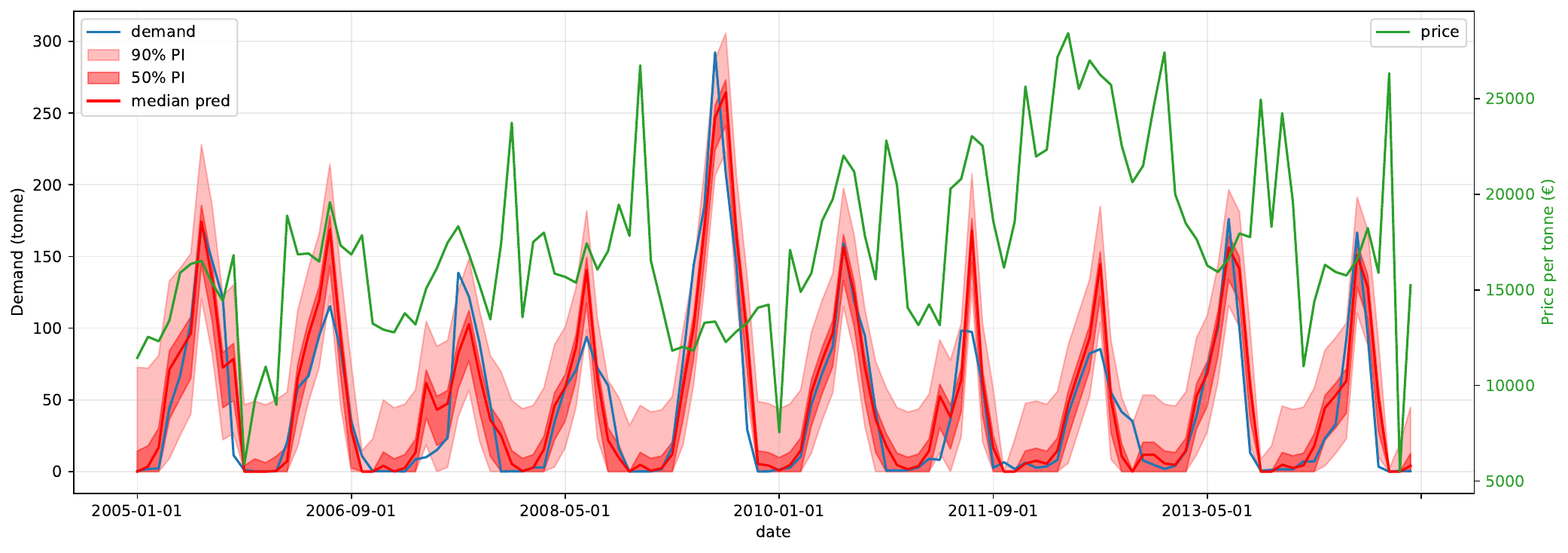}
    \caption{Example of demand and price data together with the probabilistic predictions made on the validation set used for the alignment training.}
    \label{fig:downstream_realdata_subset2_example}
\end{figure}

Following the profit objective introduced in~\cref{sec:realdata_experiment} we plug in the prices for $p_t$ where $t$ is the time period for the month and set $c_t=p_t*2.5$, $h_t=7000$. All price units are in euros.

\end{document}